%% file: main.tex
\documentclass[twoside]{article}

\input{math}

\usepackage{hyperref}       
\usepackage{url}            
\usepackage{booktabs}       
\usepackage{amsfonts}       
\usepackage{nicefrac}       
\usepackage{microtype}      
\usepackage{graphicx}       
\usepackage[round,sort]{natbib}
\usepackage{multirow}
\usepackage{bbding}
\usepackage{makecell}
\graphicspath{{media/}}     
\usepackage{enumitem}
\usepackage{threeparttable}
\usepackage{amssymb}
\usepackage{pifont}
\definecolor{mydarkgreen}{RGB}{39,130,67}
\definecolor{mydarkred}{RGB}{192,25,25}
\newcommand{\green}{\color{mydarkgreen}}
\newcommand{\red}{\color{mydarkred}}
\newcommand{\cmark}{\green\ding{51}}%
\newcommand{\xmark}{\red\ding{55}}%
\definecolor{bgcolor}{rgb}{0.93,0.99,1}
\usepackage[accepted]{aistats2024}
\usepackage{subfiles}

\usepackage[accepted]{aistats2024}
\begin{document}

%

%

\twocolumn[

\aistatstitle{An Efficient Stochastic Algorithm for Decentralized Nonconvex-Strongly-Concave Minimax Optimization}

\aistatsauthor{ Lesi Chen$^1$ \And Haishan Ye$^2$ \And  Luo Luo$^{3,*}$}
\vskip0.25cm

\aistatsaddress{\small $^1$Institute for Interdisciplinary Information Sciences, Tsinghua Univerisity \\  \small$^2$School of Management, Xi’an Jiaotong University; SGIT AI Lab, State Grid Corporation of China \\ \small $^{3}$School of Data Science, Fudan University; Shanghai Key Laboratory for Contemporary Applied Mathematics} 
]

\begin{abstract}
This paper studies the stochastic nonconvex-strongly-concave minimax optimization over a multi-agent network.
We propose an efficient algorithm, called Decentralized Recursive gradient descEnt Ascent Method (\texttt{DREAM}), which achieves the best-known theoretical guarantee for finding the $\epsilon$-stationary points. 
Concretely, it requires $\mathcal{O}(\min (\kappa^3\epsilon^{-3},\kappa^2 \sqrt{N} \epsilon^{-2} ))$ stochastic first-order oracle (SFO) calls and $\tilde \fO(\kappa^2 \epsilon^{-2})$ communication rounds, where $\kappa$ is the condition number and $N$ is the total number of individual functions.
Our numerical experiments also validate the superiority of \texttt{DREAM} over previous methods.
\end{abstract}

\section{Introduction}
This paper studies the decentralized minimax optimization problem, where $m$ agents in a network collaborate to solve the problem
\begin{align} \label{main:prob}
\min_{x\in\BR^{d_x}}\max_{ y \in \fY} f(x,y) \triangleq \frac{1}{m}\sum_{i=1}^m f_i(x,y).
\end{align}
We suppose that $f(x,y)$ is $\mu$-strongly-concave in $y$; $\fY \subseteq \BR^{d_y}$ is closed and convex; each local function on the $i$-th agent  has the following stochastic form  
\begin{align}\label{func:f_i-stochastic}
    f_i(x,y) \triangleq \BE [F_i(x,y;\xi_i)];
\end{align}
and the stochastic component $F_i(x,y;\xi_i)$ indexed by the random variable $\xi_i$ is $L$-average smooth.
The nonconvex-strongly-concave minimax problem~(\ref{main:prob}) plays an important role in many machine learning applications, such as adversarial training~\citep{nagarajan2017gradient,farnia2021train}, distributional robust optimization~\citep{sinha2017certifying,levy2020large,jin2021non},  AUC maximization~\citep{guo2020fast,liu2019stochastic,yuan2021federated}, reinforcement learning~\citep{qiu2020single,zhang2019policy,wai2018multi,jin2020efficiently}, learning with non-decomposable loss~\citep{fan2017learning,rafique2021weakly} and so on. Following existing non-asymptotic analysis for nonconvex-strongly-concave minimax optimization problems~\citep{lin2020gradient,luo2020stochastic,lin2020near}, we
focus on the task of finding an $\epsilon$-stationary point of the primal function $P(x)\triangleq\max_{y\in\fY} f(x, y)$.

In this paper, we also consider a popular special case of problem (\ref{main:prob}) when each random
variable $\xi_i$ is finitely sampled from $\{\xi_{i,1},\dots, \xi_{i,n}\}$. That is, we can write the local function as
\begin{align}\label{func:f_i-finite}
    f_i(x,y) \triangleq \frac{1}{n}\sum_{j=1}^n F_{i j}(x,y).
\end{align}
We refer to the general form (\ref{func:f_i-stochastic}) as the online (stochastic) case and refer to the special case (\ref{func:f_i-finite}) as the offline (finite-sum) case. We define $N = mn$ as the total number of individual functions for the offline case.

Nonconvex-strongly-concave minimax optimization has received increasing attentions in recent years~\citep{chen2021escaping,nagarajan2017gradient,zhang2021complexity,lin2020gradient,lin2020near,zhang2020single,luo2020stochastic,luo2022finding,zhang2022sapd+,xu2020gradient}.
In the scenario of single machine, \citet{lin2020gradient} showed the Stochastic Gradient Descent Ascent (\texttt{SGDA}) requires~$\fO(\kappa^3\epsilon^{-4})$ SFO complexity 
to find an $\epsilon$-stationary point of $P(x)$, where the condition number is defined by~$\kappa\triangleq L/\mu$.
\citet{luo2020stochastic} proposed the Stochastic Recursive gradiEnt Descent Ascent (\texttt{SREDA}), which uses the variance reduction technique of Stochastic Path Integrated Differential Estimator (\texttt{SPIDER}) \citep{fang2018spider} to establish a SFO complexity of~$\fO(\min(\kappa^3\epsilon^{-3},\sqrt{n}\kappa^2 \epsilon^{-2}))$. 
It is worth noting that the~$\fO(\min(\epsilon^{-3},\sqrt{n}\epsilon^{-2}))$ dependency on $\epsilon$ and $n$ matches the lower bound of stochastic nonconvex optimization under the average-smooth assumption~\citep{arjevani2023lower,fang2018spider}.

Distributed optimization is a popular setting for training large-scale machine learning models. It allows all agents on a given network to collaboratively optimize the global objective. In the decentralized scenario, each agent on the network only communicates with its neighbors. The decentralized training fashion
avoids the communication traffic jam on the central node~\citep{lu2021optimal}.
There have been a lot of works focusing on the complexity of decentralized stochastic minimization problems~\citep{shi2015extra,li2019communication,uribe2020dual,li2022variance,yuan2016convergence,sun2020improving,xin2022fast,li2022destress,chen2021communication,wang2021distributed,kovalev2020optimal,hendrikx2021optimal}. 

{However, the understanding of the complexity of first-order methods decentralized nonconvex-strongly-concave minimax problems is still limited.}
For offline decentralized nonconvex-strongly-concave minimax problems,
\citet{tsaknakis2020decentralized} proposed the Gradient-Tracking Descent-Ascent (\texttt{GT-DA}), which combines multi-step gradient descent ascent gradient~\citep{nouiehed2019solving} with gradient tracking~\citep{nedic2017achieving,qu2019accelerated}. \texttt{GT-DA} can provably find an $\epsilon$-stationary point of $P(x)$ within~$\tilde \fO(N \epsilon^{-2} )$ SFO calls and $\tilde \fO(\epsilon^{-2})$ communication rounds. \citet{zhang2021taming} proposed the Gradient-Tracking Gradient Descent Ascent (\texttt{GT-GDA}) which updates both variable $x$ and $y$ simultaneously. The removal of the inner loop with respect to $y$ also leads to the removal of the additional $\fO(\log({1}/{\epsilon}))$ factor in the complexity, and \texttt{GT-DA} can provably find an $\epsilon$-stationary point of $P(x)$ within $\fO(N \epsilon^{-2} )$ SFO calls and~$\fO(\epsilon^{-2})$ communication rounds.  However, both \texttt{GT-DA} and \texttt{GT-GDA} access the exact local gradients on each agent, which may be
quite expensive when $n$ is very large.
To reduce the computation complexity, \citet{zhang2021taming} further proposed the Gradient Tracking-Stochastic Recursive Variance Reduction (\texttt{GT-SRVR}) by combining \texttt{GT-GDA} with \texttt{SPIDER}~\citep{fang2018spider}. \texttt{GT-SRVR}
requires $\fO(N +  \sqrt{m N} \epsilon^{-2})$ SFO complexity to find an $\epsilon$-stationary point of $P(x)$, outperforming \texttt{GT-DA} and \texttt{GT-GDA} when $ n \gtrsim m$. 
However, for fixed~$N$ (total
number of individual functions), the SFO upper bound of \texttt{GT-SRVR} will increase with the increase of $m$ (number of agents).
This trend seems somewhat unreasonable since involving more agents in the computation intuitively should not result in a higher overall computational cost.  



For the online case, \citet{xian2021faster} introduced the variance reduction technique of STOchastic Recursive
Momentum (\texttt{STORM}) \citep{cutkosky2019momentum} and proposed the Decentralized Minimax Hybrid Stochastic Gradient Descent (\texttt{DM-HSGD}) with an upper complexity bound of $\fO( \kappa^3 \epsilon^{-3}) $ for both SFO calls and communication rounds. For this online case, 
The SFO complexity of \texttt{DM-HSGD} recovers the result of \texttt{SREDA}~\citep{luo2020stochastic} in the scenario of single-machine (when $m=1$). 
Although the SFO complexity of \texttt{DM-HSGD} is better than those for \texttt{GT-GDA} and \texttt{GT-SRVR} when $N$ has a higher order of magnitude compared to $\epsilon^{-1}$,
the communication complexity of $\fO(\epsilon^{-3})$ in \texttt{DM-HSGD} is worse than the communication complexity of $\fO(\epsilon^{-2})$ in \texttt{GT-GDA} and \texttt{GT-SRVR}. 
This implies \texttt{DM-HSGD} may has no advantage when the bottleneck is the cost of communication.


In many applications, the inner variable $y$ in minimax problem (\ref{main:prob}) is often subject to some constraints, meaning that $\fY$ typically represents a given convex set endowed with a specific model. For instance, in adversarial training \citep{goodfellow2014explaining}, the variable $y$ represents perturbations for the input, which typically lie within the box $\fY = \{ y \in \BR^{d_y}: \Vert y \Vert_{\infty} \le c \}$ for some positive constant $c$. 
In distributionally
robust optimization \citep{yan2019stochastic}, the variable~$y$ represents the probability distribution in the simplex~$\fY = \{ y \in \BR^{d_y}: \sum_k y_k = 1, y_k \ge 0 \}$. 
Dealing with the constraint on variable $y$ typically requires additional steps like projection, which lead to extra consensus error in the decentralized setting.
However, existing variance-reduced methods for decentralized nonconvex-strongly-concave minimax problems including \texttt{GT-SRVR} and \texttt{DM-HSGD} only successfully deal with the unconstrained case.

In this paper, we propose a novel method called Decentralized Recursive-gradient dEscent Ascent Method (\texttt{DREAM}) for solving decentralized nonconvex-strongly-concave minimax problems. 
We provide a unified convergence analysis for both the online and offline setups. Our analysis indicates that the proposed \texttt{DREAM} achieves the best-known complexity guarantee of both cases. We summarize the advantage of \texttt{DREAM} as follows.
\begin{itemize}[topsep=-5.5pt,itemsep=1.5pt,partopsep=1.5pt, parsep=1.5pt,leftmargin=15pt]
     \item For the offline case,  \texttt{DREAM} achieves the SFO complexity of $\fO(N + \sqrt{N} \kappa^2 \epsilon^{-2})$  and the communication complexity of $\tilde \fO(\kappa^2 \epsilon^{-2})$. 
     The SFO complexity of \texttt{DREAM} 
     is strictly better than existing methods, 
     and achieves a linear speed-up with respect to the number of agents $m$, which is better than \texttt{GT-SRVR}.
     The communication complexity of \texttt{DREAM} remains state-of-the-art, matching those of \texttt{GT-DA}/\texttt{GT-GDA}/\texttt{GT-SRVR} in the dependency of $\epsilon$, up to logarithmic factors.  
    \item For the online case, \texttt{DREAM} achieves the $\fO(\kappa^3 \epsilon^{-3})$ SFO complexity and the $\tilde \fO(\kappa^2 \epsilon^{-2})$ communication complexity.
    The SFO complexity of \texttt{DREAM} is optimal in the dependency of $\epsilon$~\citep{arjevani2023lower}. And the communication complexity of \texttt{DREAM} 
    strictly improves that of \texttt{DM-HSGD} in the dependency of both $\kappa$ and $\epsilon$.  
    \item Moreover, \texttt{DREAM} is capable of working in both unconstrained and constrained scenarios, which gives it a wider range of applicability than previous variance-reduced methods \texttt{GT-SRVR}  and \texttt{DM-HSGD}.
\end{itemize}
We compare our theoretical results with previous work in Table \ref{tab:result}. 

\paragraph{Notations.} 
Throughout this paper, we denote $\Vert \cdot \Vert$ as the Frobenius norm of a matrix or the Euclidean norm of a vector. 
We use $I_m \in \BR^{m \times m}$ to present a $m$ by $m$ identity matrix, $O_m \in \BR^{m\times m}$ to present a $m$ by $m$ zero matrix, and let $\vone=[1,\cdots,1]^\top\in\BR^m$.
We define aggregated variables $\vx\in\BR^{m\times d_x}$, $\vy\in\BR^{m\times d_y}$ and $\vz\in\BR^{m\times (d_x+d_y)}$ for all agents as
\begin{align*}
\vx=\begin{bmatrix}
\vx(1) \\ \vdots \\\vx(m)
\end{bmatrix}, &
\quad
\vy=\begin{bmatrix}
\vy(1) \\ \vdots \\\vy(m)
\end{bmatrix} 
\quad \text{and} \quad 
\vz=\begin{bmatrix}
\vz(1) \\ \vdots \\\vz(m)
\end{bmatrix},
\end{align*}
where row vectors $\vx(i)\in\BR^{d_x}$ and $\vy(i)\in\BR^{d_y}$ are local variables on the $i$-th agent; and we also denote $\vz(i) = [\vx(i); \vy(i)]\in\BR^d$ with $d=d_x+d_y$. 
We use the lowercase with the bar to represent mean vector, e.g.,
\begin{align*}
\bx = \frac{1}{m}\sum_{i=1}^m \vx(i),~
\by = \frac{1}{m}\sum_{i=1}^m \vy(i)  ~ \text{and} ~
\bz = \frac{1}{m}\sum_{i=1}^m \vz(i).
\end{align*}
Similarly, we also introduce the aggregated gradient as 
\begin{align*}
\nabla \vf(\vz)=\begin{bmatrix}
\nabla f_1(\vx(1),\vy(1))^\top \\ \vdots \\ \nabla f_m(\vx(m),\vy(m))^\top
\end{bmatrix}
\in\BR^{m\times d}.    
\end{align*}
We use $\BI[\,\cdot\,]$ to represent the indicator function of an event and define $n=+\infty$ for the online case.

\begin{table*}[t]
\centering
\caption{We compare the theoretical results of \texttt{DREAM} with previous methods for decentralized nonconvex-strongly-concave minimax optimization for both the offline and online settings. Notations $\kappa^p$ and $\kappa^q$ are used when the polynomial dependency on $\kappa$ is not explicitly provided~\citep{zhang2021taming,tsaknakis2020decentralized}. 
The notation~$\tilde \fO(\,\cdot\,)$ hides logarithmic factors in complexity.
Note that \texttt{GT-GDA} and \texttt{GT-SRVR} only consider the unconstrained problem, which corresponds to the specific case in our setting where $\fY=\BR^{d_y}$. 
The design of \texttt{DM-HSGD} includes the general constrained setting, but its convergence analysis for the constrained case looks problematic and we provide more detailed discussions in Appendix B.
}
\vskip0.3cm
\label{tab:result}
\begin{tabular}{cccccc}
    \hline
     Setup & Algorithm   & \#SFO  & \#Communication & Constraint \\
    \hline\hline
    \addlinespace
      & 
      \makecell[c]{ \texttt{GT-DA} \\
      \citet{tsaknakis2020decentralized}}
      & $\tilde \fO\left( N \kappa^p  \epsilon^{-2} \right)$ & $\tilde \fO\left( \kappa^q \epsilon^{-2} \right)$ & \cmark  \\ 
       \addlinespace
      & 
      \makecell[c]{\texttt{GT-GDA}
      \\ \citet{zhang2021taming}
      } & $\fO\left(  N  \kappa^p  \epsilon^{-2}   \right)$ & 
     $\fO\left(  \kappa^q \epsilon^{-2}   \right)$  & \xmark   \\[-0.1cm]
     Offline  \\[-0.1cm]
      &
      \makecell[c]{\texttt{GT-SRVR} \\ \citet{zhang2021taming}
      }
       & $\fO\left( N +  \sqrt{m N}   \kappa^p  \epsilon^{-2}  \right)$ & 
     $\fO\left( \kappa^q \epsilon^{-2}\right)$ & \xmark   \\
     \addlinespace 
      &
      \makecell[c]{\texttt{DREAM} (Ours) \\
      Theorem \ref{thm:main}}
       & $\fO(N + \sqrt{N } \kappa^2 \epsilon^{-2})$ & $\tilde \fO\left( \kappa^2 \epsilon^{-2}\right)$  & \cmark \\
    \addlinespace
    \hline 
    \addlinespace
    & \makecell[c]{\texttt{DM-HSGD} \\ \citet{xian2021faster}}    & $\fO \left( \kappa^3 \epsilon^{-3}   \right)$ & 
    $\fO \left( \kappa^3 \epsilon^{-3}   \right)$  & \xmark
   \\[-0.1cm]
     Online & \\[-0.1cm]
      & \makecell[c]{\texttt{DREAM} (Ours) \\
      Theorem \ref{thm:main}} & $\fO(  \kappa^3  \epsilon^{-3})$ &
     $ \tilde \fO\left( \kappa^2  \epsilon^{-2}\right)$  & 
     \cmark \\ \addlinespace
    \hline
    \end{tabular}
\end{table*}

\section{Assumptions and Preliminaries}

Throughout this paper, we suppose the stochastic NC-SC decentralized optimization problem~(\ref{main:prob}) satisfies the following standard assumptions. 
\begin{asm} \label{asm:lower-bounded}
We suppose $P(x)\triangleq \max_{y\in\fY}f(x,y)$ is lower bounded. That is, we have 
\begin{align*}
P^*=\inf_{x\in\BR^{d_x}}P(x)>-\infty.   
\end{align*}
\end{asm}

\begin{algorithm}[t]
\caption{$\FM(\va^{(0)},K)$} \label{alg:fm}
\begin{algorithmic}[1]
	\STATE \textbf{Initialize:} $\va^{(-1)}=\va^{(0)}$ \\[0.12cm]
	\STATE $\eta_a=1/\big(1+\sqrt{1-\lambda_2^2(W)}\,\big)$ \\[0.12cm]
	\STATE \textbf{for} $k = 0, 1, \dots, K$ \textbf{do}\\[0.12cm]
	\STATE\quad $\va^{(k+1)}=(1+\eta_a)W\va^{(k)}-\eta_a\va^{(k-1)}$ \\[0.12cm] 
	\STATE\textbf{end for} \\[0.12cm]
	\STATE \textbf{Output:} $\va^{(K)}$ 
\end{algorithmic}
\end{algorithm}\setlength{\textfloatsep}{0.45cm}

\begin{asm}\label{asm:smooth}
We suppose the stochastic component functions $F(x,y;\xi_i)$ on each agent is $L$-average smooth for some $L>0$. 
That is, we have
\begin{align*}
&\quad \BE\Norm{\nabla F_i(x,y;\xi_i)-\nabla F_i(x',y';\xi_i)}^2 \\
&\leq L^2\big(\Norm{x-x'}^2+\Norm{y-y'}^2\big)
\end{align*}
for any $(x,y), (x',y')\in\BR^{d_x\times d_y}$ and random index $\xi_i$.
\end{asm}


\begin{asm} \label{asm:SC}
We suppose each local function $f_i(x,y)$ is $\mu$-strongly-concave in $y$. 
That is, there exists some constant~$\mu > 0$ such that we have
\begin{align*}
    f_i(x, y) \le f_i(x, y') + \nabla_y f_i(x, y')^\top (y-y') - \frac{\mu}{2} \Vert y-y' \Vert^2
\end{align*}
for any $x \in \BR^{d_x}$ and $y,y' \in \BR^{d_y}$.
\end{asm}

Based on the smoothness and strong concavity assumptions, we can define the condition number of our optimization problem as follows.

\begin{dfn}
We define $\kappa\triangleq L/\mu$ as the condition number of problem~(\ref{main:prob}), where $L$ and $\mu$ are defined in Assumption~\ref{asm:smooth} and \ref{asm:SC} respectively.
\end{dfn}

The differentiability of the primal function $P(x)$ can be proved by Danskin's theorem~\cite[Lemma 4.3]{lin2020gradient}.

\begin{prop}\label{lem:varphi-Lip}
Under Assumptions~\ref{asm:smooth} and~\ref{asm:SC}, the function~$P(x)$ is $L_{P}$-smooth with $L_P \triangleq (\kappa+1)L$ and its gradient can be written as $\nabla P(x) = \nabla_x f(x,y^*(x))$, where we define
$y^*(x) \triangleq \arg \max_{y \in \fY} f(x,y)$.
\end{prop}

The differentiability of $P(x)$ allows us to define the $\epsilon$-stationary point.
\begin{dfn}
We call $\hat x$ an $\epsilon$-stationary point if it holds that $\Vert \nabla P(\hat x) \Vert \le \epsilon$.
\end{dfn}

The goal of algorithms is to find an $\epsilon$ stationary point of $P(x)$.
In the setting of stochastic optimization, we suppose an algorithm can get access to the stochastic first-order oracle that satisfies the following assumptions. Note that the bounded variance assumption is only required for the online case.
\begin{asm}\label{asm:SFO}
We suppose each stochastic first-order oracle (SFO) $\nabla F_i(x,y;\xi_i)$ is unbiased and has bounded variance. That is, we have
\begin{align*}
\BE[\nabla F_i(x,y;\xi_i)]=\nabla f_i(x,y)    
\end{align*}
and
\begin{align*}
    \BE\Norm{\nabla f_i(x,y)-\nabla F_i(x,y;\xi_i)}^2\leq \sigma^2 
\end{align*}
with $ \sigma^2 <  +\infty$.
\end{asm}

The mixing matrix tells us how the agents in the network communicate with their neighbors. We assume it satisfies the following assumption~\citep{song2023optimal}.

\begin{asm}\label{asm:W}
We suppose the matrix $W\in\BR^{m\times m}$ have the following properties:
\begin{enumerate}[label=\alph*.]
    \item supported on the network: $W_{i,j} \ge 0$ if and only if $i$ and $j$ are connected in the network.
    \item irreducible: $W$ cannot be conjugated into block upper triangular form by a permutation matrix.
    \item symmetric: $W=W^\top$.
    \item doubly stochastic: $W \vone = W^\top \vone = \vone$.
    \item positive semidefinite: $W \succeq O_m$.
\end{enumerate}
\end{asm}

Note that if a matrix $W$ satisfies Assumption \ref{asm:W}  a-d, we can let Assumption \ref{asm:W}e be automatically satisfied by choosing $(W + I_m) /2$ to be the new mixing matrix.

By the Perron–Frobenius theorem, the eigenvalues of $W$ can be sorted by 
\begin{align*}
0 \le \lambda_m(W) \le \cdots \le \lambda_2(W) < \lambda_1(W) =1.     
\end{align*}
We then define the spectral gap of $W$ as follows.

\begin{dfn} \label{dfn:spectral-gap}
    For a matrix $W$ that satisfies Assumption \ref{asm:W}, we define the spectral gap as $\delta \triangleq 1-\lambda_2(W)$.
\end{dfn}

It is well known that the spectral gap of $W$ is related to the mixing rate on the network.

\begin{prop}[{\citet[Lemma 16]{koloskova2019decentralized}}]
Given a matrix $W$ that satisfies Assumption \ref{asm:W}, for any vector $\va \in \BR^{m \times d}$, for the standard mixing iterate given by $\va^{(k+1)} = W \va^{(k)}$,
we have
  \begin{align*}
      \Vert W \va - \vone \bar \va \Vert \le (1-\delta)^K \Vert \va - \vone \bar \va \Vert,
  \end{align*}
where $\Norm{\,\cdot\,}$ is the Frobenius norm.
\end{prop}


This simple mixing strategy is adopted by previous works including \citet{tsaknakis2020decentralized,xian2021faster,zhang2021taming}, but it would lead to an unavoidable communication complexity dependency of at least $\fO({1}/{\delta})$, which is suboptimal in the dependency of $\delta$.
To accelerate the mixing rate, we introduce the \FM~ sub-procedure \citep{liu2011accelerated}, which can lead to the optimal $\fO({1}/{\sqrt{\delta}}\,)$ dependency. 

\begin{prop}[{\citet[Proposition 1]{ye2020multi}}] \label{prop:fm}
Given a matrix $W$ that satisfies Assumption \ref{asm:W}, running 
Algorithm~\ref{alg:fm} ensures $\frac{1}{m}\vone^\top\va^{(K)} = \bar{a}^{(0)}$ and
\begin{align*}
\big\|\va^{(K)}-\vone\bar{a}^{(0)}\big\| \leq  c_1 \big( 1 - c_2 \sqrt{\delta}\,\big)^K \big\|\va^{(0)}-\vone\bar{a}^{(0)}\big\|,
\end{align*}
where $\bar{a}^{(0)} = \frac{1}{m}\vone^\top\va^{(0)}$, $\Norm{\,\cdot\,}$ is the Frobenius norm,  $c_1 = \sqrt{14}$ and $c_2 = 1 - 1/\sqrt{2}$. 
\end{prop}

To tackle the possible constraint in $y$,
we define the projection and {the constrained reduced gradient} ~\citep{nesterov2018lectures}.
\begin{dfn}
We define 
\begin{align*}
\small\begin{split}    
\Pi(y) \triangleq \argmin_{y' \in \fY}  \Vert y' - y \Vert^2 ~\text{and}~{\bf\Pi}(\vy)\triangleq \argmin_{\vy'(i) \in \fY}  \Vert \vy' - \vy \Vert^2
\end{split}
\end{align*}
for $y\in\BR^{d_y}$ and $\vy\in\BR^{m\times d_y}$ respectively.
\end{dfn}

\begin{dfn} \label{dfn:gmap}
We also define {the constrained reduced gradient}  of $f$ at $(x, y)$ with respect to $y$ as 
\begin{align*}
    G_{\eta}(x,y) = \frac{\Pi(y+ \eta \nabla_y f(x,y)) - y}{\eta}
\end{align*}
with some $0 < \eta  \le 1/L$.
\end{dfn}

\begin{algorithm*}[htbp]
\caption{Decentralized Recursive-gradient dEscent Ascent Method (\texttt{DREAM})} \label{alg:DREAM}
\small
\begin{algorithmic}[1]
\STATE \textbf{Notations:} Let $\vz_t = [\vx_t, \vy_t]\in\BR^{m\times d}$ and $\vs_t = [\vu_t, \vv_t]\in\BR^{m\times d}$. 
\\[0.15cm]
\STATE \textbf{Input:} initial parameter $\bz_0\in\BR^d$, stepsize $\eta>0$, stepsize ratio $\gamma \in (0,1]$,  probability $p\in(0,1]$, small mini-batch size $b$, large mini-batch size $b'$ (we set $b'=n$ for the offline case), initial communication rounds $K_0$, small communication rounds $K$, large communication rounds $K'$. \\[0.15cm]
\STATE $\vz_0=\vone \bz_0$ \\[0.15cm]
\STATE\textbf{parallel for} $i = 1, \dots, m$ \textbf{do}  \\[0.15cm]
\STATE \quad Sample $ \displaystyle {\fS_0(i)=  
\begin{cases}
    \{\xi_{i,1},\dots,\xi_{i,b'}\} ~ \text{i.i.d.}, & \text{online case}; \\[0.15cm]
    \{\xi_{i,1},\dots,\xi_{i,n}\}, & \text{offline case}; \\[0.1cm]
\end{cases}
}$ \\[0.2cm]
\STATE \quad  $\displaystyle{\vg_0(i)=\dfrac{1}{b'}\sum_{\xi_{i,j}\in \fS_0(i)}\nabla F_i(\vz_0(i);\xi_{i,j})}$ \\[0.15cm]
\STATE\textbf{end parallel for} \\[0.15cm]
\STATE $\vs_0= \FM(\vg_0,K_0)$ \\[0.15cm]
\STATE \textbf{for} $t = 0, \dots, T-1$ \textbf{do}\\[0.15cm]
\STATE \quad Sample $\zeta_t \sim {\rm Bernoulli}(p)$ \\[0.15cm]
\STATE\quad $\vx_{t+1} = \FM\big( \vx_t - \gamma\eta \vu_t,K\big)$ \\[0.15cm]
\STATE\quad $\vy_{t+1} = \FM\big( {\bf\Pi}(\vy_t + \eta \vv_t), K\big)$ \\[0.15cm]
\STATE\quad \textbf{parallel for} $i = 1, \dots, m$ \textbf{do}\\[0.15cm]
\STATE \quad \quad \textbf{if } $\zeta_t = 1$ \textbf{do} \\[0.15cm] 
\STATE \quad \quad \quad Sample $ \displaystyle {\fS_t'(i)=  
\begin{cases}
    \{\xi_{i,1},\dots,\xi_{i,b'}\} ~ \text{i.i.d.}, & \text{online case}; \\[0.15cm]
    \{\xi_{i,1},\dots,\xi_{i,n}\}, & \text{offline case}; \\[0.1cm]
\end{cases}
}$ \\[0.2cm]
\STATE \quad \quad \quad $ \displaystyle{\vg_{t+1}(i) = \dfrac{1}{b'}\sum_{\xi_{i,j}\in \fS_t'(i)}\!\!\nabla F_i(\vz_{t+1}(i);\xi_{i,j})} $ \\[0.15cm]
\STATE \quad \quad \textbf{else} \\[0.15cm]
\STATE \quad \quad \quad Sample $\omega_t(i) \sim {\rm Bernoulli}(q)$ \\[0.15cm]
\STATE \quad \quad \quad Sample $\fS_t(i) = \{\xi_{i,1},\dots,\xi_{i,b}\} $ i.i.d. \\[0.15cm]
\STATE \quad \quad \quad $\displaystyle{\vg_{t+1}(i) = \vg_t(i) + \dfrac{\omega_t(i)}{bq}\sum_{\xi_{i,j}\in \fS_t(i)}\!\!\big(\nabla F_i(\vz_{t+1}(i);\xi_{i,j}) - \nabla F_i(\vz_t(i);\xi_{i,j})\big)}$ \\[0.15cm]
\STATE \quad \quad \textbf{end if} \\[0.15cm]
\STATE\quad\textbf{end parallel for} \\[0.15cm]
\STATE \quad $\displaystyle{\vs_{t+1} =
\begin{cases}
\begin{aligned}
\FM \big(\vs_{t}+\vg_{t+1}-\vg_{t},K'\big), \quad \quad & \text{if } \zeta_t = 1;  \\[0.1cm]
\FM\big(\vs_{t}+\vg_{t+1}-\vg_{t},K\big), \quad \quad & \text{if } \zeta_t = 0; \\[0.1cm]
\end{aligned}
\end{cases}
} $ \\[0.15cm]
\STATE\textbf{end for} \\[0.15cm]
\STATE\textbf{Output:} $x_{\rm out}$  by uniformly sampling from $\{\vx_{0}(1), \vx_0(2), \cdots, \vx_{T-1}(m)  \}$ \\[0.15cm]
\end{algorithmic}
\end{algorithm*}

\section{The Proposed Algorithm}

In this section, we propose a novel stochastic algorithm named Decentralized Recursive-gradient dEscent Ascent Method (\texttt{DREAM}) for decentralized nonconvex-strongly-concave minimax problems.  
We provide a unified framework for analyzing our \texttt{DREAM} for both online and offline cases.
It shows the algorithm can find an $\epsilon$-stationary point of $P(x)$ within at most $\fO(\kappa^3 \epsilon^{-3})$ and $\fO(N + \sqrt{ N} \kappa^2 \epsilon^{-2})$ SFO calls for the online and offline setting respectively; and both of two settings require at most~$ \fO(\kappa^2 \epsilon^{-2} \log m/ \sqrt{\delta})$ communication rounds. 

\subsection{ Method Overview}

\texttt{DREAM} constructs stochastic recursive gradients $\vg_t(i)$ \citep{fang2018spider,nguyen2017sarah,li2021page,luo2020stochastic} to estimate the local gradients, \textit{i.e.} $\vg_t(i) \approx \nabla f_i(\vz_t(i)) $. 
This step (Line 22 in Algorithm~\ref{alg:DREAM}) reduces the variance of stochastic gradient estimators on each agent, leading to the optimal $\fO(\epsilon^{-3})$ dependency in the SFO upper complexity bound. 
\texttt{DREAM} then 
applies the gradient tracking technique~\citep{qu2019accelerated,di2016next} to track the average gradient over the network via the gradient tracker  $\vs_t(i)$, \textit{i.e.} $\vs_t(i) \approx \nabla f(\vz_t(i)) $. The gradient tracker $\vs_t(i)$ is also updated in a recursive way (Line 24 in Algorithm~\ref{alg:DREAM}). This step is commonly used to achieve convergence when the data distribution on each agent does not satisfy the i.i.d assumption in decentralized optimization~\citep{nedic2017achieving,song2023optimal}. With the gradient tracker, \texttt{DREAM} applies the two-timescale gradient descent ascent~\citep{lin2020gradient} to solve the maximization of $y$ and minimization of $x$ simultaneously (Line 13-14 in Algorithm~\ref{alg:DREAM}).
The random variable $\zeta_t \sim {\rm Bernoulli}(p)$ (Line 12 in Algorithm~\ref{alg:DREAM}) decides the batch size and the number of consensus steps in the iteration. By appropriately choosing the probability~$p$, batch sizes~$b,b'$, and consensus steps~$K,K'$, \texttt{DREAM}  achieves the best-known computation complexity and communication complexity trade-off in decentralized nonconvex-strongly-concave minimax optimization.




\subsection{A Novel Lyapunov Function}

Different from previous works~\citep{zhang2021taming,xian2021faster,luo2020stochastic}, we propose a novel Lyapunov function as follows:
\begin{align} \label{dfn:Phi}
\begin{split}
    \Phi_t &\triangleq \Psi_t  + \frac{1}{m\eta} C_t + \frac{\eta}{m p} V_t + \frac{\eta}{p} U_t, \quad {\rm where} \\
    \Psi_t &= P(\bx_t) - P^* + \alpha ( P(\bx_t) - f(\bx_t,\by_t)),  \\
    C_t &= \Vert \vz_t - \vone \bz_t \Vert^2 + \eta^2 \Vert \vs_t - \vone \bar s_t \Vert^2 \\
    V_t &=  \frac{1}{m} \Vert \mathbf{g}_t - \nabla \mathbf{f} (\mathbf{z}_t) \Vert^2 ~~\text{and}~~ \\
    U_t &= \left \Vert \frac{1}{m} \sum_{i=1}^m  (\vg_t(i) - \nabla f_i(\vz_t(i))) \right \Vert^2.
\end{split}
\end{align}
We set $\alpha \in (0,1]$ in later analysis.
Below, we illustrate the meaning of each quantity in the Lyapunov function.
\begin{itemize}[left=1.5em]
\item $\Psi_t$ measures the optimization error. The gradient descent ascent step ensures that $\Psi_t$ can decrease monotonically at each iteration ~\citep{yang2020global,chen2022faster}.
\item $C_t$ measures the consensus error, which can be bounded by the properties of gradient tracking and consensus steps~\citep{song2023optimal}.
\item $V_t$ and $U_t$ measure the variance of $\vg_t$, which can be bounded by the property of martingale~\citep{fang2018spider,luo2020stochastic}.
We provide a detailed discussion for the roles of $V_t$ and~$U_t$ after Lemma~\ref{lem:Lyp-V} and \ref{lem:Lyp-U}.
\end{itemize}

We can show that $C_t$, $V_t$, and $U_t$ are sufficiently small during the iterations, which allows us to analyze \texttt{DREAM} by analogizing gradient descent ascent on the mean variables~$\bx_t$ and $\by_t$ with some small noise.
\begin{rmk}
Our Lyapunov function $\Phi_t$ characterizes the sub-optimality of the maximization problem $\max_{y\in\fY} f(\bx_t,y)$ by $P(\bx_t) - f(\bx_t,\by_t)$, which is easy to be analyzed for stochastic variance reduced algorithm in the decentralized setting.
In contrast, \citet{xian2021faster} and \citet{zhang2021taming} measure the sub-optimality by~$\Vert \by_t-y^*(\bx_t) \Vert^2$, which leads to their analysis be more complicated than ours.
\end{rmk}

\subsection{Convergence Analysis}
Our analysis starts from the following descent lemma.
\begin{lem}\label{lem:Lyp-P}
For Algorithm \ref{alg:DREAM}, we set the parameters by $\eta \le 1/(4L)$ and 
\begin{align} \label{choice:lambda}
    \gamma= \frac{\alpha}{(1+\alpha)128 \kappa^2}.
\end{align}
Then for any $\alpha>0$ it holds that
\begin{align*} 
\BE[\Psi_{t+1}] &\le  \Psi_t - \frac{\gamma\eta}{2} \Vert \nabla P(\bar x_t) \Vert^2 - \frac{1}{8 \gamma\eta} \BE\Vert \bx_{t+1} - \bx_t \Vert^2 \\
&- \frac{\alpha}{16 \eta} \BE\Vert \bz_{t+1} - \bz_t \Vert^2 + 6 \alpha \eta U_t + \frac{3 \alpha }{m \eta} C_t.
\end{align*}
\end{lem}

If both $C_t$ and $U_t$ are sufficiently small, the above lemma indicates that the optimization error decreases by roughly $(\gamma\eta / 2) \Vert \nabla P(\bx_t) \Vert^2$ at each step in expectation and the remaining proof can follow the gradient descent~\citep{nesterov2018lectures,bubeck2015convex} on the primal function $P(x)$.

Recall Proposition \ref{prop:fm}.
We further define
the discount factors of consensus error that arises from mixing  steps (Line 10, 13, 14 and 24 in Algorithm \ref{alg:DREAM}) as:
\begin{align}\label{dfn:rho}
\begin{split}
    \rho_0 &= c_1\big(1-c_2 \sqrt{\delta}\big)^{K_0}, \\
\rho' &= c_1\big(1-c_2\sqrt{\delta}\big)^{K'}, \\
\rho &= c_1\big(1-c_2 \sqrt{\delta}\big)^K.
\end{split}
\end{align}

Then we can bound the consensus error as follows.
\begin{lem}\label{lem:Lyp-C}
For Algorithm \ref{alg:DREAM}, let  $\rho^2\le1/24$. Then 
\begin{align*}
\BE[C_{t+1}] \le&  12 c\rho^2  C_t + 6 \rho'^2 \eta^2 m V_t  +2c\rho^2 m \mathbb{E}\Vert \bar z_{t+1} - \bar z_t \Vert^2 \\
&\quad +\frac{ 6 \rho'^2 m \eta^2  \sigma^2}{b'}  \mathbb{I}[b' < n],
\end{align*}
where we define $c = \max \{ 1/(bq),1 \} $.
\end{lem}
\begin{rmk}
For convenience, we define $n=+\infty$ for the online case and $b'=n$ for the offline case. Then
\begin{align*}
\mathbb{I}[b' < n] =
\begin{cases}
1, & \text{for online case}, \\
0, & \text{for offline case}. 
\end{cases}
\end{align*}
This notation allows us to present the analysis for both cases in one unified framework.
\end{rmk}

Lemma~\ref{lem:Lyp-P} and \ref{lem:Lyp-C} mean the decrease of the optimization error~$\Psi_t$ and consensus error $C_t$ requires the reasonable upper bounds of $V_t$ and $U_t$, which can be characterized by the following recursions.

\begin{lem}\label{lem:Lyp-V}
For Algorithm \ref{alg:DREAM},  we have
\begin{align*} 
\BE[V_{t+1}] &\le (1-p) V_t  + \frac{4 (1-p) L^2}{mbq} C_t \\
&\quad +  \frac{p\sigma^2}{b'} \BI[b' <n ]  + \frac{3 (1-p)L^2 }{bq} \BE\Vert \bz_{t+1} - \bz_t \Vert^2.
\end{align*}
\end{lem}

\begin{lem}
\label{lem:Lyp-U}
For Algorithm \ref{alg:DREAM},  we have
\begin{align*} 
\BE[U_{t+1}] &\le (1-p) U_t  + \frac{4 (1-p) L^2}{m^2 bq} C_t \\
&\quad +  \frac{p\sigma^2}{m b'} \BI[b' <n ]  + \frac{3 (1-p)L^2 }{m bq} \BE\Vert \bz_{t+1} - \bz_t \Vert^2.
\end{align*}
\end{lem}

Note that for any vector sequence $a_1,\cdots,a_m$ we always have $\Vert \sum_{i=1}^m  a_i \Vert^2 \le m \sum_{i=1}^m \Vert a_i \Vert^2$, which directly implies $U_t \le V_t$. Therefore one can use the quantity $ V_t$ only in the analysis as \citet{zhang2021taming}, but the separation of $U_t$ and $V_t$ makes our bound tighter. As a consequence, we show a linear speed-up in the SFO complexity with respect to the number of agents $m$ which was not shown by \citet{zhang2021taming}.

Putting Lemma \ref{lem:Lyp-P}, \ref{lem:Lyp-C} and \ref{lem:Lyp-V} together, we can prove the main result for \texttt{DREAM} as follows.

\begin{thm}\label{thm:main}
For Algorithm \ref{alg:DREAM} , we set parameters 
\begin{align} \label{para:val}
\begin{split}
& \eta = \frac{1}{ 48 L},\quad   b = \left \lceil \sqrt{\frac{b'}{m}} ~\right \rceil, \quad q = \frac{1}{b} \sqrt{\frac{b'}{m}}, \quad  p = \frac{b q }{bq+b'} \\
&T = \left \lceil  \frac{16 \Psi_0}{\gamma\eta \epsilon^2}+ \frac{2}{p} \right \rceil, \quad K_0 = \left \lceil  \frac{\log \left( 16 c_1/ (\gamma m \epsilon^2)\right)}{c_2 \sqrt{\delta}} \right \rceil, \\
& K = \left \lceil \frac{5 \log (c_1(m / b' +1))}{c_2 \sqrt{\delta}} \right \rceil, \quad K' = \left \lceil \frac{5 \log (c_1 m) }{c_2 \sqrt{\delta}} \right \rceil,
\end{split}
\end{align} 
where $c_1,c_2$ are defined in Proposition \ref{prop:fm}, $\gamma = \Theta(\kappa^{-2})$ follows (\ref{choice:lambda}), $\alpha = 1/8$ and 
\begin{align} \label{choice:b}
    b' = \begin{cases}
    \left \lceil 32 \sigma^2 / (\gamma m  \epsilon^2) \right \rceil, & \text{for online case}, \\
    n, & \text{for offline case}.
    \end{cases}
\end{align}
Then the output satisfies $\BE \Vert \nabla P(x_{\rm out}) \Vert \le \epsilon$ within the overall SFO complexity 
\begin{align*}
\begin{cases}
\fO\big( \kappa^2 \sigma^2 \epsilon^{-2} + \kappa^3 L \sigma \epsilon^{-3}\big), & \text{for online case}; \\[0.1cm]
\fO\big(mn + \sqrt{mn} \kappa^2 L \epsilon^{-2}\big), & \text{for offline case};
\end{cases} 
\end{align*}
and communication complexity
$\fO(({\kappa^2 L \epsilon^{-2}\log m})/\sqrt{\delta}\,)$.
\end{thm}

This theorem shows that \texttt{DREAM} achieves the best-known complexity guarantee both in computation and communication, either for online or offline cases.


\begin{figure*}[t]\centering
\begin{tabular}{ccc}
\includegraphics[scale=0.325]{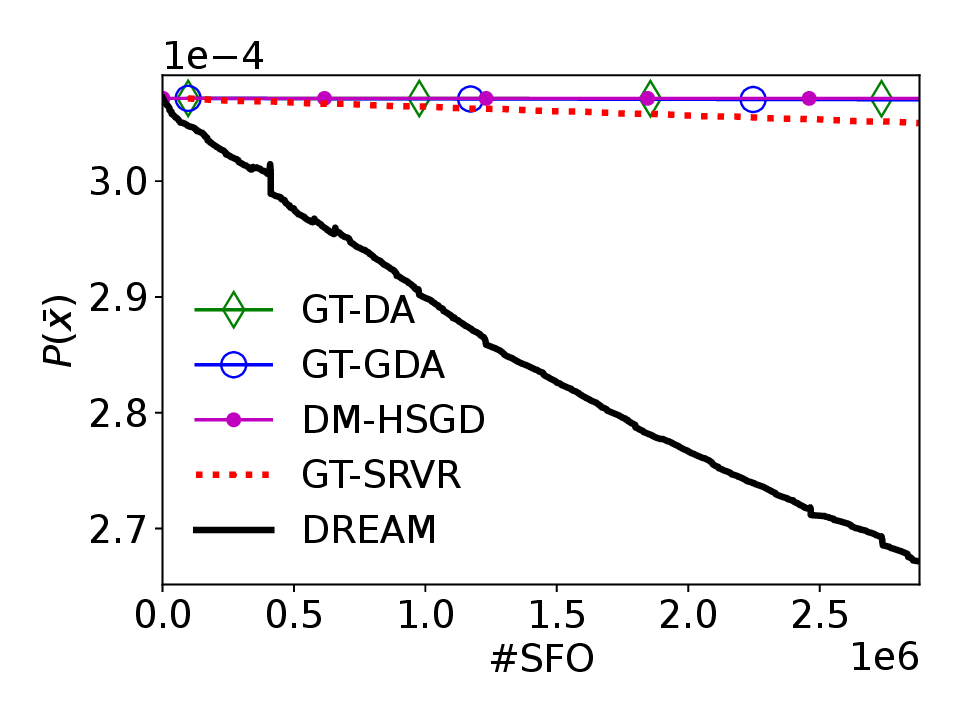} &
\includegraphics[scale=0.325]{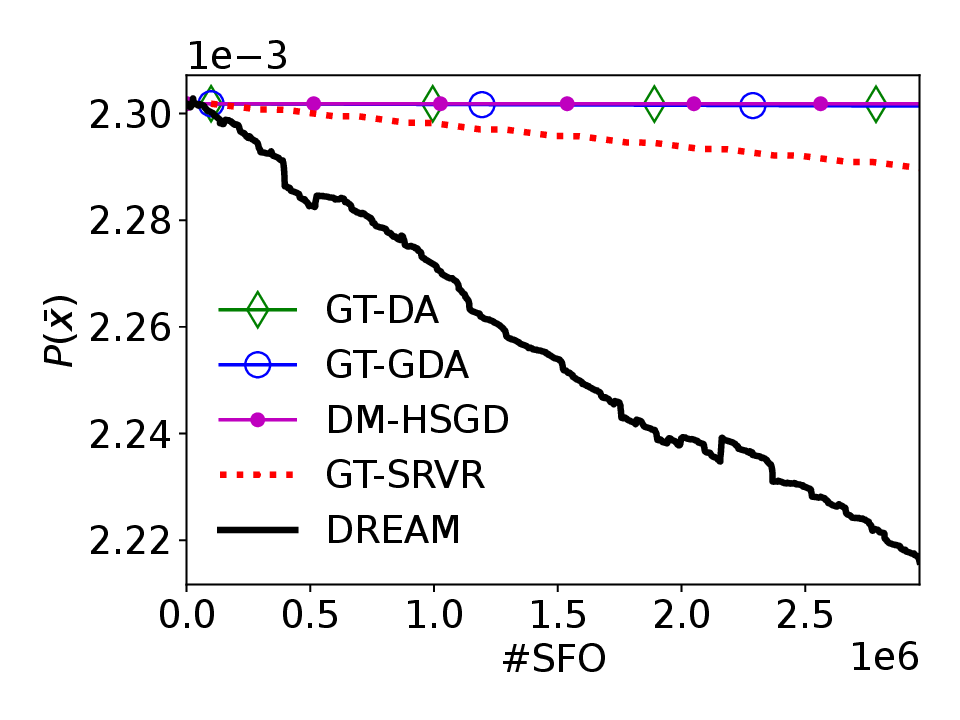} &
\includegraphics[scale=0.325]{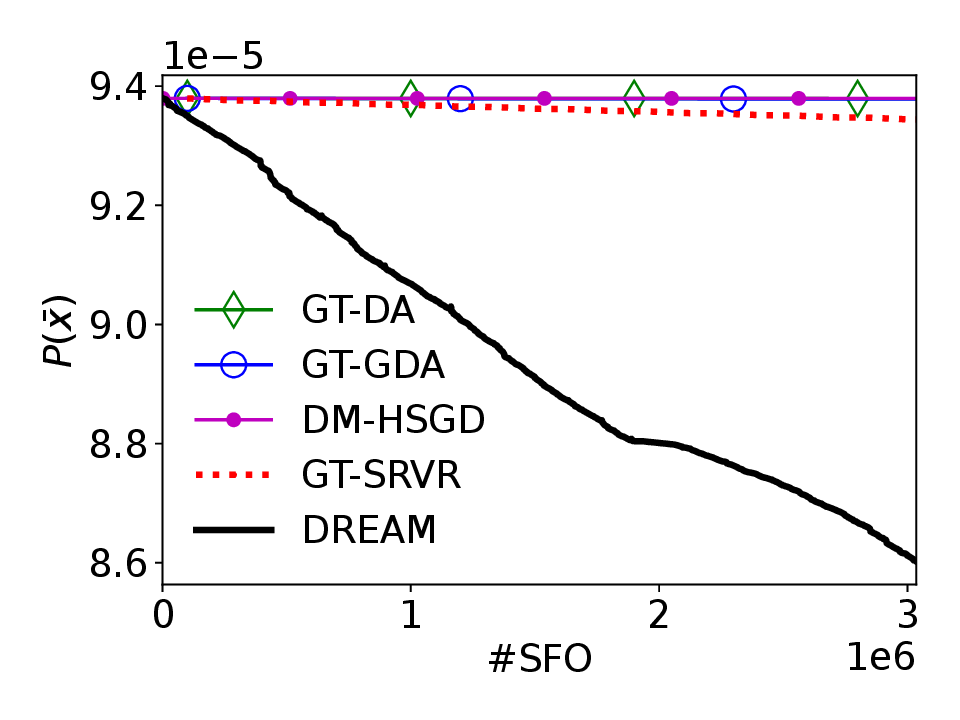} \\
(a) a9a & (b) w8a & (c) ijcnn1
\end{tabular}\vskip-0.1cm
\caption{Comparison on the number of SFO calls against $  P(\bar x) \triangleq  \max_{y \in \Delta_N} f(\bar x,y)$.} 
\label{fig:sfo}
\end{figure*}\vskip0.5cm

\begin{figure*}[t]\centering
\begin{tabular}{ccc}
\includegraphics[scale=0.325]{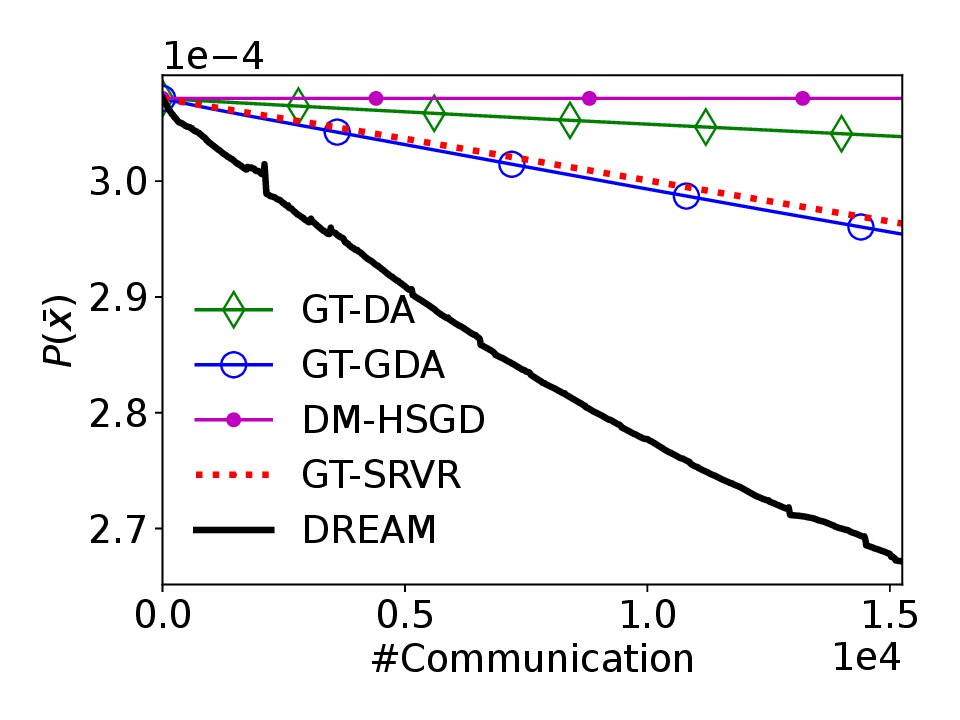} &
\includegraphics[scale=0.325]{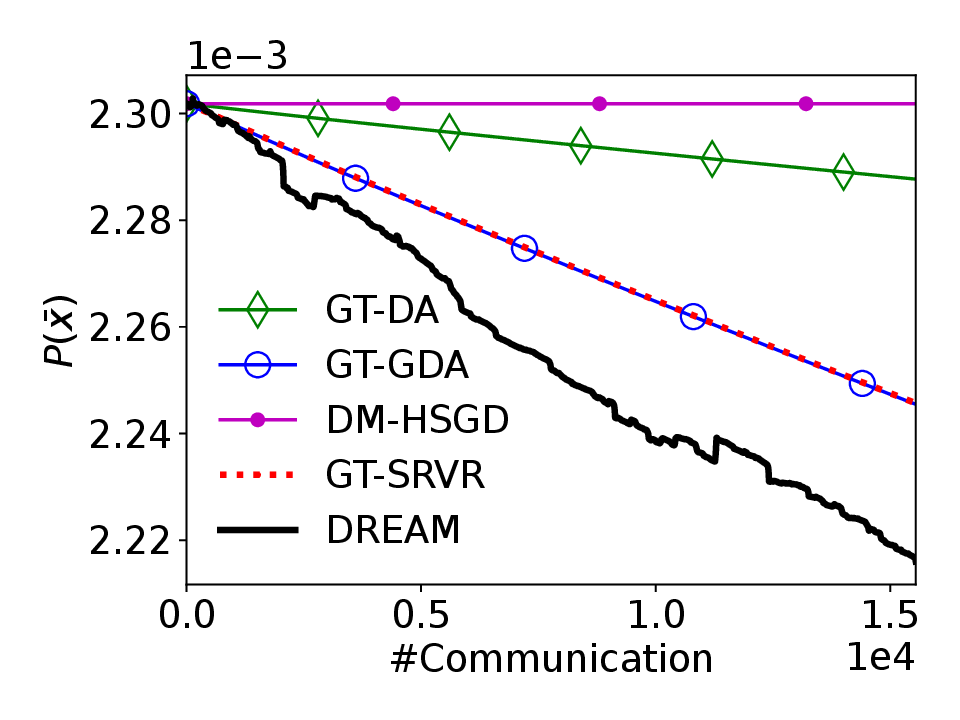} &
\includegraphics[scale=0.325]{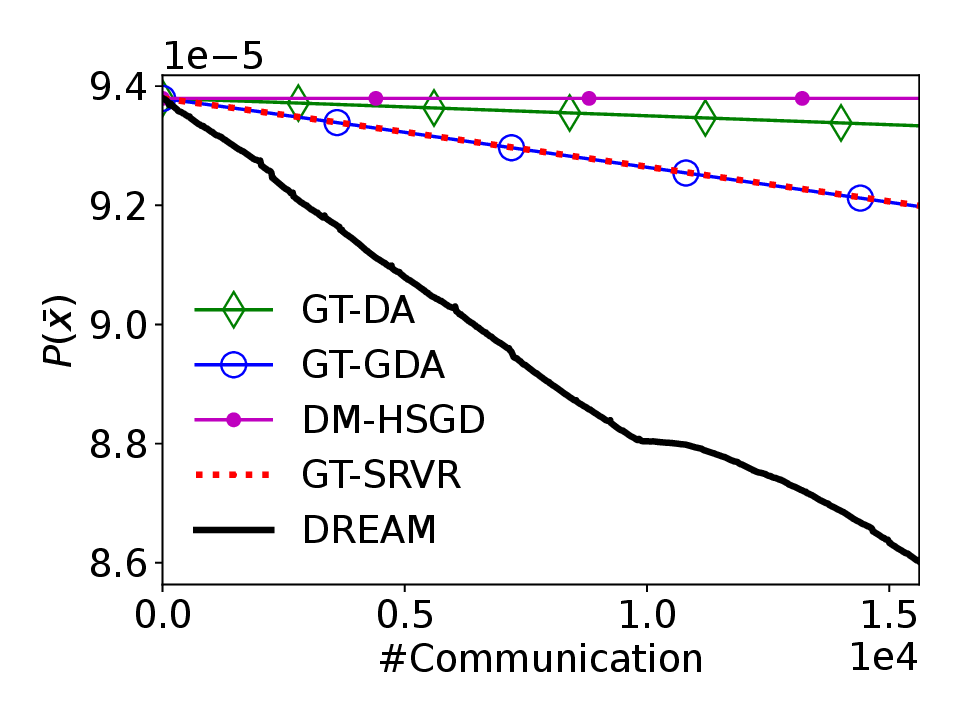} \\
(a) a9a & (b) w8a & (c) ijcnn1
\end{tabular}\vskip-0.1cm
\caption{Comparison on the number of communication rounds against $ P(\bar x) \triangleq  \max_{y \in \Delta_N} f(\bar x,y)$.} 
\label{fig:comm}
\end{figure*}\vskip0.5cm

\section{Experiments} \label{sec:exp}

We conduct numerical experiments on the model of robust logistic regression~\citep{luo2020stochastic,yan2019stochastic}. The source codes are available \footnote{https://github.com/TrueNobility303/DREAM}.  
It considers training the binary classifier $x\in\BR^d$ on 
dataset $\{(a_k, b_k)\}_{k=1}^{N}$, where $a_k\in\BR^{d}$ is the feature of the $k$-th sample and $b_k \in \{1,-1 \}$ is the corresponding label.
The total $N$ samples are equally distributed on $m$ agents, with $n = N / m$ samples on each agent. 
The decentralized minimax optimization problem of this model corresponds to formulation (\ref{main:prob})
with local functions
\begin{align*}
    f_i(x) \triangleq \frac{1}{n}\sum_{i=1}^n \big( y_{ij} l_{ij}(x) - V(y) + g(x) \big),
\end{align*}
where $l_{ij}(x) = \log(1+\exp(b_{ij} a_{ij}^\top x))$,
\begin{align*}
\small\begin{split}    
   & V(y) = \frac{1}{2 N^2} \Vert N y-\vone \Vert^2,~
   g(x) = \theta \sum_{k=1}^d \frac{\nu x_k^2}{1 + \nu x_k^2},~  \theta = 10^{-5}, \\
   & \nu=10~~\text{and}~~\Delta_N=\Big\{y\in\BR^N:y_k\in[0,1], \sum_{k=1}^N y_k=1\Big\}.
\end{split}
\end{align*}
We set mixing matrix $W$ as the $\pi$-lazy random walk matrix on a ring graph with 8 nodes ($m=8$) and let $\tau = 0.999$, which leads to a low consensus rate since~$\lambda_2(W) = \tau + (1-\tau) \cos\left( {2 \pi (m-1)}/{m}\right) $. 

We compare the performance of our proposed method \texttt{DREAM} with \texttt{DM-HSGD} \citep{xian2021faster},  \texttt{GT-DA}~\citep{tsaknakis2020decentralized}, \texttt{GT-GDA} and \texttt{GT-SRVR} \citep{zhang2021taming} on 
datasets ``a9a'', ``w8a'' and ``ijcnn1''~\citep{chang2011libsvm}.
All of the algorithms are implemented by MPI to simulate the distributed training scenario.

For \texttt{DREAM}, we tune $b,b'$ from $\{64,128,256, 512\}$, tune $p,q$ from $\{0.2,0.5,0.9 \}$ and tune $K_0,K,K'$ from $\{2,5,10\}$. For \texttt{DM-HSGD} , we set 
momentum parameter $\beta =0.01$ by following \citet{xian2021faster}, and 
tune $b,b_0$ from $\{64,128,256,512\}$.
For \texttt{GT-SRVR}, we let the epoch length $Q = \left \lceil \sqrt{n} \,\right \rceil $ as suggested by the authors.
For \texttt{GT-DA}, we let the number of inner loops $R = 4$ by following \citet{tsaknakis2020decentralized}.
For each algorithm, we tune $\eta$ and $\gamma$ from $\{1,0.1,0.01,0.001\}$ and  $\{0.1,0.01,0.001,0.0001\}$  respectively.

We present the empirical results for the comparisons of computation complexity and communication complexity in Figure \ref{fig:sfo} and \ref{fig:comm}. It is shown that the proposed \texttt{DREAM} performs apparently better than baselines. 

\section{Conclusion and Future Work}

In this paper, we have proposed an efficient algorithm \texttt{DREAM} for decentralized stochastic nonconvex-strongly-concave minimax optimization.
The theoretical analysis showed \texttt{DREAM} achieves the best-known SFO complexity of $\mathcal{O}(\min (\kappa^3\epsilon^{-3},\kappa^2 \sqrt{N} \epsilon^{-2} ))$ and communication complexity of $\tilde \fO( \kappa^2 \epsilon^{-2})$. 
The numerical experiments on robust logistic regression show the empirical performance of \texttt{DREAM} is obviously better than state-of-the-art methods. We discuss possible future directions and some subsequent works in Appendix H.

\subsubsection*{Acknowledgements}

Luo has been supported by National Natural Science Foundation of China (No. 62206058), Shanghai Sailing Program (22YF1402900), and Shanghai Basic Research Program (23JC1401000).
Ye has been supported by the National Natural Science Foundation of China (No. 12101491).

\bibliography{reference}
\bibliographystyle{plainnat}


\section*{Checklist}

The checklist follows the references. For each question, choose your answer from the three possible options: Yes, No, Not Applicable.  You are encouraged to include a justification to your answer, either by referencing the appropriate section of your paper or providing a brief inline description (1-2 sentences). 
Please do not modify the questions.  Note that the Checklist section does not count towards the page limit. Not including the checklist in the first submission won't result in desk rejection, although in such case we will ask you to upload it during the author response period and include it in camera ready (if accepted).

\textbf{In your paper, please delete this instructions block and only keep the Checklist section heading above along with the questions/answers below.}

 \begin{enumerate}

 \item For all models and algorithms presented, check if you include:
 \begin{enumerate}
   \item A clear description of the mathematical setting, assumptions, algorithm, and/or model. [Yes]
   \item An analysis of the properties and complexity (time, space, sample size) of any algorithm. [Yes]
   \item (Optional) Anonymized source code, with specification of all dependencies, including external libraries. [Yes]
 \end{enumerate}

 \item For any theoretical claim, check if you include:
 \begin{enumerate}
   \item Statements of the full set of assumptions of all theoretical results. [Yes]
   \item Complete proofs of all theoretical results. [Yes]
   \item Clear explanations of any assumptions. [Yes]     
 \end{enumerate}

 \item For all figures and tables that present empirical results, check if you include:
 \begin{enumerate}
   \item The code, data, and instructions needed to reproduce the main experimental results (either in the supplemental material or as a URL). [Yes]
   \item All the training details (e.g., data splits, hyperparameters, how they were chosen). [Yes]
         \item A clear definition of the specific measure or statistics and error bars (e.g., with respect to the random seed after running experiments multiple times). [Yes]
         \item A description of the computing infrastructure used. (e.g., type of GPUs, internal cluster, or cloud provider). [Yes]
 \end{enumerate}

 \item If you are using existing assets (e.g., code, data, models) or curating/releasing new assets, check if you include:
 \begin{enumerate}
   \item Citations of the creator If your work uses existing assets. [Yes]
   \item The license information of the assets, if applicable. [Not Applicable]
   \item New assets either in the supplemental material or as a URL, if applicable. [Not Applicable]
   \item Information about consent from data providers/curators. [Not Applicable]
   \item Discussion of sensible content if applicable, e.g., personally identifiable information or offensive content. [Not Applicable]
 \end{enumerate}

 \item If you used crowdsourcing or conducted research with human subjects, check if you include:
 \begin{enumerate}
   \item The full text of instructions given to participants and screenshots. [Not Applicable]
   \item Descriptions of potential participant risks, with links to Institutional Review Board (IRB) approvals if applicable. [Not Applicable]
   \item The estimated hourly wage paid to participants and the total amount spent on participant compensation. [Not Applicable]
 \end{enumerate}
 \end{enumerate}
\input{supp}

\end{document}

%% file: math.tex
\usepackage{amsmath}\allowdisplaybreaks
\usepackage{amsfonts,bm}
\usepackage{amssymb}
\usepackage{amsthm}
\usepackage{amsmath}
\usepackage{algorithm}
\usepackage{algorithmic}
\usepackage{multirow}
\usepackage{booktabs}
\usepackage{xcolor}

\def\eqref#1{equation~(\ref{#1})}

\def\1{\bf{1}}

\newcommand{\Norm}[1]{\left\| #1 \right\|}

\def\eps{{\epsilon}}

\def\vone{{\bf{1}}}

\def\va{{\bf{a}}}

\def\vf{{\bf{f}}}
\def\vg{{\bf{g}}}

\def\vs{{\bf{s}}}

\def\vu{{\bf{u}}}
\def\vv{{\bf{v}}}

\def\vx{{\bf{x}}}
\def\vy{{\bf{y}}}
\def\vz{{\bf{z}}}

\def\fO{{\mathcal{O}}}

\def\fS{{\mathcal{S}}}

\def\fY{{\mathcal{Y}}}

\def\BE{{\mathbb{E}}}

\def\BI{{\mathbb{I}}}

\def\BR{{\mathbb{R}}}

\DeclareMathOperator*{\argmax}{arg\,max}
\DeclareMathOperator*{\argmin}{arg\,min}

\newtheorem{thm}{Theorem}[section]
\newtheorem{dfn}{Definition}[section]

\newtheorem{lem}{Lemma}[section]
\newtheorem{asm}{Assumption}[section]
\newtheorem{rmk}{Remark}[section]

\newtheorem{prop}{Proposition}[section]

\makeatletter
\def\Ddots{\mathinner{\mkern1mu\raise\p@
\vbox{\kern7\p@\hbox{.}}\mkern2mu
\raise4\p@\hbox{.}\mkern2mu\raise7\p@\hbox{.}\mkern1mu}}
\makeatother

\makeatletter
\newcommand*{\rom}[1]{\expandafter\@slowromancap\romannumeral #1@}
\makeatother

\def\FM {{{\texttt{FastMix}}}}
\def\bg {{{\bar g}}}
\def\bs {{{\bar s}}}
\def\bv {{{\bar v}}}
\def\bx {{{\bar x}}}
\def\by {{{\bar y}}}

\def\bu {{{\bar u}}}
\def\bz {{{\bar z}}}

%% file: supp.tex
\onecolumn
\aistatstitle{Supplementary Material: \\An Efficient Stochastic Algorithm for Decentralized Nonconvex-Strongly-Concave Minimax Optimization}

\appendix
\section{Some Useful Lemmas}\label{appendix:lemmas}

We first provide some useful lemmas.


\begin{lem}\label{lem:norm}
For any $a_1,\dots,a_m\in\BR^d$, we have 
\begin{align*}
\Norm{\frac{1}{m}\sum_{i=1}^m a_i}^2 \leq \frac{1}{m} \sum_{i=1}^m \Norm{a_i}^2.
\end{align*}
\end{lem}

\begin{lem}\label{lem:avg-norm}
For any matrix $\vz\in\BR^{m\times d}$ and $\bz=\frac{1}{m}\vone^\top\vz$, we have
\begin{align*}
\Norm{\vz - \vone\bz} \leq \Norm{\vz},
\end{align*}
\end{lem}

\begin{lem}\label{lem:smooth-b}
Under Assumption~\ref{asm:smooth}, we have
\begin{align*}
\Norm{\nabla\vf(\vz)-\nabla \vf(\vz')} \leq  L\Norm{\vz-\vz'}    
\end{align*}
for any $\vz, \vz'\in\BR^{m\times d}$.
\end{lem}

\begin{lem}\label{lem:stst12}
For Algorithm \ref{alg:DREAM}, we have
$\bs_t=\frac{1}{m}\vone^\top\vg_t=\bg_t$.
\end{lem}
\begin{proof}
We prove this lemma by induction.
For $t=0$, we have
\begin{align*}
\bs_0 =  \frac{1}{m}\vone^\top\vs_0 = \frac{1}{m}\vone^\top\vg_0.
\end{align*}
Suppose the statement holds for $t\leq k$. Then for $t=k+1$, the induction base means that
\begin{align*}
  & \bs_{k+1} \\
= & \bs_{k} + \bg_{k+1} - \bg_{k}  \\
= & \frac{1}{m}\vone^\top\vg(\vz_{k}) + \bg_{k+1} - \frac{1}{m}\vone^\top\vg(\vz_{k})  \\
= & \bg_{k+1}  \\
= & \frac{1}{m}\vone^\top\vg_{k+1},
\end{align*}
which finished the proof.
\end{proof}

\begin{lem} \label{lem:error-grad}
Under Assumption~\ref{asm:smooth}, for Algorithm \ref{alg:DREAM}, it holds that
\begin{align*}
    \Vert \bs_t - \nabla f(\bz_t) \Vert^2 \le 2 \left \Vert \frac{1}{m} \sum_{i=1}^m  (\vg_t(i) - \nabla f_i(\vz_t(i))) \right \Vert^2 + \frac{2L^2}{m} \Vert \vz_t - \vone \bz_t \Vert^2.
\end{align*}

\end{lem}

\begin{proof}
Using Young's inequality and Assumption~\ref{asm:smooth}, we have
\begin{align*}
&  \Norm{\bs_t-\nabla f(\bz_t)}^2 \\
=& \Norm{\frac{1}{m}\sum_{i=1}^m \big(\vg_t(i)-\nabla f_i(\bz_t)\big)}^2 \\ 
\le& 2 \left \Vert \frac{1}{m} \sum_{i=1}^m  (\vg_t(i) - \nabla f_i(\vz_t(i))) \right \Vert^2 + 2 \Norm{\frac{1}{m} \sum_{i=1}^m (\nabla f_i(\vz_t(i)) - \nabla f_i(\bz_t)) }^2 \\
\le & 2 \left \Vert \frac{1}{m} \sum_{i=1}^m  (\vg_t(i) - \nabla f_i(\vz_t(i))) \right \Vert^2 + \frac{2}{m} \sum_{i=1}^m \Vert \nabla f_i(\vz_t(i)) - \nabla f_i(\bz_t) \Vert^2 \\
\le 2 &\ \left \Vert \frac{1}{m} \sum_{i=1}^m  (\vg_t(i) - \nabla f_i(\vz_t(i))) \right \Vert^2 + \frac{2L^2}{m}\Norm{\vz_t-\vone\bz_t}^2. 
\end{align*}
\end{proof}


\begin{lem} \label{lem:hypergrad}
When $f(x,y)$ is $L$-smooth and $\mu$-strongly-concave in $y$, it holds that
\begin{align*}
    \Vert \nabla P(x) - \nabla_x f(x,y) \Vert \le 2 \kappa \Vert G_{\eta}(x,y) \Vert,
\end{align*}
where $G_{\eta}(x,y) $ denotes {the constrained reduced gradient} 
\begin{align*}
    G_{\eta}(x,y) = \frac{\Pi(y+ \eta \nabla_y f(x,y)) - y}{\eta}
\end{align*}
with any $\eta \le 1/L$.
\end{lem}
\begin{proof}
Using the Danskin's theorem, i.e. $\nabla P(x) = \nabla_x f(x,y^*(x))$ as well as the $L$-smoothness, we have
\begin{align}\label{ieq:smooth-grad_x}
    \Vert \nabla P(x) - \nabla_x f(x,y) \Vert \le L \Vert y^*(x) - y \Vert.
\end{align}
The Corollary 1 of \citet[Appendix A]{luo2020stochastic} tells us
\begin{align}\label{ieq:mapping}
    \mu \Vert y - y^*(x) \Vert \le 2 \Vert G_{\eta} (x,y) \Vert.
\end{align}
We prove this lemma by combing inequalities (\ref{ieq:smooth-grad_x}) and (\ref{ieq:mapping}).
\end{proof}

\section{Discussions on DM-HSGD} \label{apx:DMHSGD}

Decentralized Minimax Hybrid Stochastic Gradient Descent (DM-HSGD) considers the minimax problem (\ref{main:prob}) in both unconstrained and constrained cases~\citep{xian2021faster}. 
However, the theoretical analysis of this method only works for unconstrained case, and it is problematic for the general constrained case.\footnote{The theoretical issue of DM-HSGD has also been mentioned by \citet{zhang2023jointly}.} 

The analysis of DM-HSGD heavily relies on its Lemma 5 \citep[Appendix A.2]{xian2021faster}.
However, the proof of this result is based on the following equation \citep[Equation (29)]{xian2021faster}\footnote{The analysis of DM-HSGD uses $\eta_y$ and $\bar u_t$ to present the stepsize and gradient estimator with respect to $y$, which play the similar roles of $\eta$ and $\bar v_t$ in our notations.} 
\begin{align}\label{eq:incorrect}
    2\eta_y  \langle  \bar u_t, \hat y_t - \bar y_t \rangle = \Vert \bar y_t -\hat y_t \Vert^2 + \Vert \bar y_{t+1} - \bar y_t \Vert^2 - \Vert \bar y_{t+1} - \hat y_t \Vert^2,
\end{align}
where $\hat y_t=\argmax_{y\in\fY}f(\bar x_t,y)$ corresponds to $y^*(\bar x_t)$ in our notations.
Note that Equation (\ref{eq:incorrect}) does not hold for general convex and compact set $\fY$. We further illustrate this below.

We denote $\vy_{t+1/2} = \vy_t + \eta_y \vu_t$ by following \cite{xian2021faster}'s notation. 
The procedure of DM-HSGD guarantees that $\bar y_{t+1} = \bar \Pi (\vy_{t+1/2} )$, then we have
\begin{align}\label{eq:correct}
    2\eta_y   \langle \bar u_t, \hat y_t - \bar y_t \rangle = 2 \langle \bar y_{t+1/2} - \by_t, \hat y_t - \by_t \rangle = \Vert \bar y_t - \hat y_t \Vert^2 + \Vert \bar y_{t+1/2} - \bar y_t \Vert^2 - \Vert \bar y_{t+1/2} - \hat y_t \Vert^2.
\end{align}
The only difference between equations (\ref{eq:incorrect}) and  (\ref{eq:correct}) is their last terms.
In the unconstrained case that $\fY=\BR^{d_y}$, it holds that 
\begin{align}\label{eq:eq}
    \bar y_{t+1} = \bar \Pi (\vy_{t+1/2}) = \bar y_{t+1/2},
\end{align}
which leads to
\begin{align*}
\Vert \bar y_{t+1} - \hat y_t \Vert^2 = \Vert \bar y_{t+1/2} - \hat y_t \Vert^2.
\end{align*}
However, the equation (\ref{eq:eq}) may not hold in the constrained case, since we can not guarantee $\bar \Pi (\vy) = \by$ in general. 
For example, consider that
\begin{align*}
\fY = \{ y \in \BR^2: \Vert y \Vert^2 =2 \},  \qquad \vy(1) = (2,2) 
\qquad\text{and}\qquad \vy(2) = (2,-2),     
\end{align*}
then we can have $\by = (2,0)$ while $\bar \Pi(\vy) = \left(\sqrt{2},0 \right)$.

For the analysis of projected first-order methods for constrained problems (even for minimization problems), we typically introduce { the constrained reduced gradient} 
\cite[Definition 2.2.3]{nesterov2018lectures} and apply the first-order optimality condition (such as Equation (\ref{eq:optimality-proj}) in Appendix \ref{appendix:C}) to establish the convergence results. 
However, such popular techniques are not included in the analysis of DM-HSGD.

In decentralized setting, the extension from unconstrained case to constrained case is non-trivial, since the projection step introduces additional consensus error. 
As a consequence, our analysis is essentially different from the  \citet{xian2021faster}. 
We design a novel Lyapunov function and present the details in Lemma \ref{lem:Lyp-P} to address this issue.
 


\section{The Proof of Lemma \ref{lem:Lyp-P}}\label{appendix:C}

\begin{proof}
Denote $ \mathbf{y}_{t+1}' = {\bf \Pi}( \mathbf{y}_t + \eta \vv_t )$. By Proposition \ref{prop:fm}, the update rules of  $\vx_t,\vy_t$ means
\begin{align}\label{update:avg}
\begin{split}
\bx_{t+1} &=  \bx_t - \gamma \eta \bu_t \qquad \text{and}\qquad \by_{t+1} = \by_{t+1}' = \bar{\Pi} (\vy_t + \eta \vv_t).
\end{split}
\end{align}
In the view of inexact gradient descent on $P(x)$, it yields
\begin{align} \label{eq:varphi}
\begin{split}
&\quad \mathbb{E}[P(\bar x_{t+1})] \\ &\le \mathbb{E} \left [ P(\bar x_t) + \nabla P(\bar x_t)^\top (\bar x_{t+1} - \bar x_t) + \frac{L_{P} }{2} \Vert \bar x_{t+1} - \bar x_t \Vert^2 \right] \\
&= \mathbb{E} \left [ P(\bar x_t) -\gamma \eta \nabla P(\bar x_t)^\top \bar u_t  + \frac{L_{P} \gamma^2 \eta^2 }{2} \Vert \bar u_t \Vert^2 \right] \\
&= \mathbb{E} \left [ P(\bar x_t) - \frac{\gamma \eta}{2} \Vert\nabla P(\bar x_t) \Vert^2 - \left( \frac{\gamma \eta}{2} - \frac{L_{P} \gamma^2 \eta^2 }{2}  \right)  \Vert \bar u_t \Vert^2+ \frac{\gamma \eta}{2} \Vert \nabla P(\bar x_t) - \bar u_t \Vert^2 \right] \\
&\le \mathbb{E} \left [ P(\bar x_t) - \frac{\gamma\eta}{2} \Vert\nabla P(\bar x_t) \Vert^2 - \left( \frac{\gamma\eta}{2} - \frac{L_{P} \gamma^2 \eta^2 }{2}  \right)  \Vert \bar u_t \Vert^2 \right] \\
&\quad + \mathbb{E} \left[ \frac{\gamma \eta}{m} \sum_{i=1}^m \Vert \nabla P(\bar x_t) -  \nabla_x f(\bar x_t,\mathbf{y}_t(i)) \Vert^2 + \frac{\gamma\eta}{m} \sum_{i=1}^m\Vert \bar u_t - \nabla_x f(\bar x_t,\mathbf{y}_t(i)) \Vert^2 \right] \\
&\le \mathbb{E} \left [ P(\bar x_t) - \frac{\gamma\eta}{2} \Vert\nabla P(\bar x_t) \Vert^2 - \left( \frac{\gamma\eta}{2} - \frac{L_{P} \gamma^2 \eta^2 }{2}  \right)  \Vert \bar u_t \Vert^2 \right] \\
&\quad + \mathbb{E} \left[ \frac{4 \kappa^2\gamma \eta}{m} \sum_{i=1}^m \Vert G_{\eta}(\bar x_t, \vy_t(i)) \Vert^2 + \frac{\gamma\eta}{m} \sum_{i=1}^m\Vert \bar u_t - \nabla_x f(\bar x_t,\vy_t(i)) \Vert^2 \right] \\
&\le \mathbb{E} \left [ P(\bar x_t) - \frac{\gamma\eta}{2} \Vert\nabla P(\bar x_t) \Vert^2 - \left( \frac{\gamma\eta}{2} - \frac{L_{P} \gamma^2 \eta^2 }{2}  \right)  \Vert \bar u_t \Vert^2 \right] \\
&\quad + \mathbb{E} \left[ \frac{8 \kappa^2\gamma \eta}{m} \sum_{i=1}^m \Vert \mathbf{v}_t'(i) \Vert^2 + \frac{8 \kappa^2\gamma \eta}{m} \sum_{i=1}^m \Vert \mathbf{v}_t(i) - \nabla_y f(\bar x_t, \mathbf{y}_t(i)) \Vert^2+ \frac{\gamma\eta}{m} \sum_{i=1}^m\Vert \bar u_t - \nabla_x f(\bar x_t,\mathbf{y}_t(i)) \Vert^2 \right],
\end{split}
\end{align}
where $\mathbf{v}_t'(i) = (\mathbf{y}_{t+1}'(i) - \mathbf{y}_t(i)) / \eta$. Above, the first inequality follows from the smoothness of $P(x)$ by Proposition \ref{lem:varphi-Lip}; the second one is due to the Young's inequality; the third inequality holds according to Lemma \ref{lem:hypergrad}; in the last one we use the Young's inequality along with
\begin{align*} 
    &\quad \Vert \vv_t'(i) - G_{\eta}(\bar x_t, \vy_t(i)) \Vert \\
    &= \frac{1}{\eta} \Vert \Pi( \vy_t+ \eta \vv_t(i)) - \Pi(\vy_t + \eta \nabla_y f(\bx_t, \vy_t(i))) \Vert \\
    &\le \Vert \vv_t(i) - \nabla_y f(\bx_t, \vy_t(i)) \Vert.
\end{align*}
Recall that $\vy_{t+1}' = {\bf \Pi}(\vy_t + \eta \vv_t)$ and the first-order optimality of the projection,  we have
\begin{align} \label{eq:optimality-proj} 
    (\mathbf{y}_t(i) + \eta  \mathbf{v}_t(i) - \mathbf{y}'_{t+1}(i))^\top (y - \mathbf{y}'_{t+1}(i)  ) &\le 0
\end{align}
for any $y \in \mathcal{Y}$.
Taking $y = \bar y_t$ in above inequality, we get
\begin{align}\label{eq:proj}
\begin{split}
&\quad \mathbf{v}_t(i)^\top ( \bar y_t - \mathbf{y}'_{t+1}(i))\\
&\le \frac{1}{\eta} (\bar y_t - \mathbf{y}'_{t+1}(i))^\top(\mathbf{y}'_{t+1}(i) - \mathbf{y}_t(i))  \\
&= \frac{1}{2 \eta} \Vert \bar y_t -  \mathbf{y}_t(i) \Vert^2 -\frac{1}{2\eta} \Vert \bar y_t - \mathbf{y}'_{t+1}(i) \Vert^2 - \frac{1}{2 \eta} \Vert \mathbf{y}'_{t+1}(i) - \mathbf{y}_t(i) \Vert^2. 
\end{split}
\end{align}
In the view of inexact gradient descent ascent on $f(x,y)$, it yields
\begin{align*}
\begin{split}
&\quad \mathbb{E}[- f(\bar x_{t+1}, \mathbf{y}'_{t+1}(i))] \\
&\le \mathbb{E}\left[ -f(\bar x_{t+1}, \mathbf{y}_t(i)) - \nabla_y f(\bar x_{t+1},\mathbf{y}_t(i))^\top (\mathbf{y}'_{t+1}(i) - \mathbf{y}_t(i))+ \frac{L}{2} \Vert \mathbf{y}'_{t+1}(i) -  \mathbf{y}_t(i) \Vert^2 \right] \\
&\le \mathbb{E} \left[ - f(\bar x_t,\mathbf{y}_t(i)) -  \nabla_x f(\bar x_t,\mathbf{y}_t(i))^\top (\bar x_{t+1} - \bar x_t) + \frac{L}{2} \Vert \bar x_{t+1} - \bar x_t \Vert^2 \right] \\
&\quad + \mathbb{E} \left[ - \nabla_y f(\bar x_{t+1},\mathbf{y}_t(i))^\top (\mathbf{y}'_{t+1}(i) - \mathbf{y}_t(i))+ \frac{L}{2} \Vert \mathbf{y}'_{t+1}(i) - \mathbf{y}_t(i) \Vert^2\right] \\
&\le \mathbb{E} \left[ - f(\bar x_t,\bar y_t)+\nabla_y f(\bar x_t, \mathbf{y}_t(i))^\top (\bar y_t - \mathbf{y}_t(i)) -  \nabla_x f(\bar x_t,\mathbf{y}_t(i))^\top (\bar x_{t+1} - \bar x_t) + \frac{L}{2} \Vert \bar x_{t+1} - \bar x_t \Vert^2 \right] \\
& \quad + \mathbb{E} \left[ - \nabla_y f(\bar x_{t+1},\mathbf{y}_t(i))^\top (\mathbf{y}'_{t+1}(i) - \mathbf{y}_t(i))+ \frac{L}{2} \Vert \mathbf{y}'_{t+1}(i) - \mathbf{y}_t(i) \Vert^2\right] \\
&= \mathbb{E} \left[ - f(\bar x_t,\bar y_t)+\nabla_y f(\bar x_t, \mathbf{y}_t(i))^\top (\bar y_t - \mathbf{y}'_{t+1}(i)) -  \nabla_x f(\bar x_t,\mathbf{y}_t(i))^\top (\bar x_{t+1} - \bar x_t) + \frac{L}{2} \Vert \bar x_{t+1} - \bar x_t \Vert^2 \right] \\
&\quad +  \mathbb{E} \left[ (\nabla_y f(\bar x_t, \mathbf{y}_t(i))- \nabla_y f(\bar x_{t+1},\mathbf{y}_t(i)))^\top (\mathbf{y}'_{t+1}(i) - \mathbf{y}_t(i))+ \frac{L}{2} \Vert \mathbf{y}'_{t+1}(i) - \mathbf{y}_t(i) \Vert^2\right] \\
&= \mathbb{E} \left[ - f(\bar x_t,\bar y_t)+\mathbf{v}_t(i)^\top (\bar y_t - \mathbf{y}'_{t+1}(i)) -  \nabla_x f(\bar x_t,\mathbf{y}_t(i))^\top (\bar x_{t+1} - \bar x_t) + \frac{L}{2} \Vert \bar x_{t+1} - \bar x_t \Vert^2 \right] \\
&\quad +  \mathbb{E} \left[ (\nabla_y f(\bar x_t, \mathbf{y}_t(i))- \nabla_y f(\bar x_{t+1},\mathbf{y}_t(i)))^\top (\mathbf{y}'_{t+1}(i) - \mathbf{y}_t(i))+ \frac{L}{2} \Vert \mathbf{y}'_{t+1}(i) - \mathbf{y}_t(i) \Vert^2\right] \\
&\quad + \mathbb{E} \left[ (\nabla_y f(\bar x_t, \mathbf{y}_t(i))- \mathbf{v}_t(i))^\top (\bar y_t -\mathbf{y}'_{t+1}(i)) \right] \\
&\le \mathbb{E} \left[ - f(\bar x_t,\bar y_t)- \frac{1}{2 \eta}\Vert \mathbf{y}'_{t+1}(i) - \mathbf{y}_t(i) \Vert^2 -  \nabla_x f(\bar x_t,\mathbf{y}_t(i))^\top (\bar x_{t+1} - \bar x_t) + \frac{L}{2} \Vert \bar x_{t+1} - \bar x_t \Vert^2 \right] \\
&\quad +  \mathbb{E} \left[ (\nabla_y f(\bar x_t, \mathbf{y}_t(i))- \nabla_y f(\bar x_{t+1},\mathbf{y}_t(i)))^\top (\mathbf{y}'_{t+1}(i) - \mathbf{y}_t(i))+ \frac{L}{2} \Vert \mathbf{y}'_{t+1}(i) - \mathbf{y}_t(i) \Vert^2\right] \\
&\quad + \mathbb{E} \left[ \frac{\eta}{2} \Vert\nabla_y f(\bar x_t, \mathbf{y}_t(i))- \mathbf{v}_t(i)) \Vert^2 + \frac{1}{2 \eta} \Vert \bar y_t -  \mathbf{y}_t(i) \Vert^2\right] \\
&\le \mathbb{E} \left[ - f(\bar x_t,\bar y_t)- \left(\frac{\eta}{4} - \frac{\eta^2 L}{2}\right)\Vert \mathbf{v}'_{t}(i)  \Vert^2  +\left( \frac{\gamma^2 \eta^2 L}{2} + \gamma^2 \eta^3 L^2 + \frac{3 \gamma^2 \eta}{2}\right) \Vert \bar u_t \Vert^2 \right] \\
&\quad +\mathbb{E} \left[  \frac{\eta}{2} \Vert \nabla_x f(\bar x_t,\mathbf{y}_t(i)) - \bar u_t \Vert^2 + \frac{\eta}{2} \Vert\nabla_y f(\bar x_t, \mathbf{y}_t(i))- \mathbf{v}_t(i)) \Vert^2 + \frac{1}{2 \eta} \Vert \bar y_t -  \mathbf{y}_t(i)
\Vert^2\right],
\end{split}
\end{align*}
where $\vv_t'(i) = (\vy_{t+1}'(i) - \vy_t(i)) / \eta$.  Above, the equations rely on rearranging the terms; the first two inequalities are based on the $L$-smoothness of the objective $f(x,y)$. The third one follows from the concavity of in the direction of $y$; the second last inequality is a use of (\ref{eq:proj}) and the Young's inequality; the last inequality follows from the Young's inequality and the $L$-smoothness that leads to
\begin{align*}
&\quad (\nabla_y f(\bx_t, \vy_t(i))- \nabla_y f(\bx_{t+1},\vy_t(i)))^\top (\vy_{t+1}'(i) - \vy_t(i)) \\
&\le \frac{1}{4 \eta} \Vert \vy_{t+1}'(i) - \vy_t(i) \Vert^2 + \eta \Vert \nabla_y f(\bx_{t+1}, \vy_t(i)) - \nabla_y f(\bx_t, \vy_t(i)) \Vert \\
&\le \frac{1}{4 \eta} \Vert \vy_{t+1}'(i) - \vy_t(i) \Vert^2 + \eta L^2 \Vert \bx_{t+1} - \bx_t \Vert^2 \\
&= \frac{\eta}{4} \Vert \vv_t'(i) \Vert^2 + \gamma^2 \eta^3 L^2 \Vert \bu_t \Vert^2.
\end{align*}

Using the concavity of $f(x,y)$ in variable $y$ as well as the Jensen's inequality, we have
\begin{align} \label{eq:fxy}
\begin{split}
& \mathbb{E}[-f(\bar x_{t+1}, \bar y'_{t+1})] \\
\le & \quad \mathbb{E} \left[ - f(\bar x_t,\bar y_t)- \left(\frac{\eta}{4} - \frac{\eta^2 L}{2}\right)\frac{1}{m} \sum_{i=1}^m \Vert \mathbf{v}'_{t}(i)  \Vert^2  +\left( \frac{\gamma^2 \eta^2 L}{2} + \gamma^2 \eta^3 L^2 + \frac{3 \gamma^2 \eta}{2}\right) \Vert \bar u_t \Vert^2 \right] \\
&\quad + \frac{1}{m} \sum_{i=1}^m \mathbb{E} \left[  \frac{\eta}{2} \Vert \nabla_x f(\bar x_t,\mathbf{y}_t(i)) - \bar u_t \Vert^2 + \frac{\eta}{2} \Vert\nabla_y f(\bar x_t, \mathbf{y}_t(i))- \mathbf{v}_t(i)) \Vert^2 + \frac{1}{2 \eta} \Vert \bar y_t -  \mathbf{y}_t(i) \Vert^2\right].
\end{split}
\end{align}

Note that $\bar y_{t+1}' = \bar y_{t+1}$. Adding (\ref{eq:varphi}) multiplying $(1+\alpha)$ along with (\ref{eq:fxy}) multiplying $\alpha$, we get
\begin{align*}
& \BE[P_{t+1}] \\
\le & P_t + (1+\alpha) \mathbb{E} \left [  - \frac{\gamma\eta}{2} \Vert\nabla P(\bar x_t) \Vert^2 - \left( \frac{\gamma\eta}{2} - \frac{L_{P} \gamma^2 \eta^2 }{2}  \right)  \Vert \bar u_t \Vert^2 \right]+ (1+\alpha) \times \\
&\quad \mathbb{E} \left[\frac{8\kappa^2\gamma \eta}{m} \sum_{i=1}^m \Vert \mathbf{v}_t'(i) \Vert^2 + \frac{8 \kappa^2\gamma \eta}{m} \sum_{i=1}^m \Vert \mathbf{v}_t(i)- \nabla_y f(\bar x_t, \mathbf{y}_t(i)) \Vert^2+ \frac{\gamma\eta}{m} \sum_{i=1}^m\Vert \bar u_t - \nabla_x f(\bar x_t,\mathbf{y}_t(i)) \Vert^2 \right] \\
& + \alpha \left[ - \left(\frac{\eta}{4} - \frac{\eta^2 L}{2}\right)\frac{1}{m} \sum_{i=1}^m \Vert \mathbf{v}'_{t}(i)  \Vert^2  + \left(\frac{\gamma^2 \eta^2 L}{2} + \gamma^2 \eta^3 L^2 + \frac{3 \gamma^2 \eta}{2}\right) \Vert \bar u_t \Vert^2 \right] \\
& + \frac{\alpha}{m} \sum_{i=1}^m \mathbb{E} \left[  \frac{\eta}{2} \Vert \nabla_x f(\bar x_t,\mathbf{y}_t(i)) - \bar u_t \Vert^2 + \frac{\eta}{2} \Vert\nabla_y f(\bar x_t, \mathbf{y}_t(i))- \mathbf{v}_t(i)) \Vert^2 + \frac{1}{2 \eta} \Vert \bar y_t -  \mathbf{y}_t(i) \Vert^2\right].
\end{align*}
Rearranging the above result leads to
\begin{align*}
\mathbb{E}[\Psi_{t+1}] &\le \mathbb{E} \left[ \Psi_t - \frac{(1+\alpha)\gamma \eta}{2} \Vert \nabla P(\bar x_t) \Vert^2\right] \\
&\quad - \left( (1+\alpha) \left(\frac{\gamma \eta}{2} - \frac{L_{P} \gamma^2 \eta^2}{2} \right) - \alpha \left( \frac{L\gamma^2 \eta^2 }{2} +\gamma^2 \eta^3 L  + \frac{3 \gamma^2 \eta}{2} \right) \right) \mathbb{E}[  \Vert \bar u_t \Vert^2 ] \\
&\quad - \left( \alpha \left( \frac{\eta}{4} - \frac{L \eta^2}{2} \right) -(1+\alpha) 8 \kappa^2 \gamma \eta\right) \frac{1}{m} \sum_{i=1}^m\mathbb{E}[   \Vert \mathbf{v}_t'(i) \Vert^2 ] \\
&\quad + \left((1+\alpha)\gamma + \frac{\alpha}{2}\right) \frac{\eta}{m} \sum_{i=1}^m \mathbb{E}[ \Vert \bar u_t  - \nabla_x f(\bar x_t, \mathbf{y}_t(i) ) \Vert^2] \\
&\quad + \left((1+\alpha) 8 \kappa^2 \gamma + \frac{\alpha}{2} \right) \frac{ \eta}{m}  \sum_{i=1}^m\mathbb{E}[ \Vert \mathbf{v}_t(i) - \nabla_y f(\bar x_t, \mathbf{y}_t(i) ) \Vert^2 ] \\
&\quad + \frac{\alpha}{2 \eta m} \sum_{i=1}^m \BE[ \Vert \by_t - \vy_t(i) \Vert^2].
\end{align*}
Since $\alpha \in (0,1]$, taking $\eta \le 1/(4L)$ and the definition of $\gamma$ in (\ref{choice:lambda}) mean
\begin{align*}
     \alpha \left( \frac{\eta}{4} - \frac{L \eta^2}{2} \right) -(1+\alpha) 8 \kappa^2 \gamma \eta  \ge \frac{\alpha \eta}{16}
\end{align*}
as well as 
\begin{align*}
    (1+\alpha) \left(\frac{\gamma \eta}{2} - \frac{L_{P} \gamma^2 \eta^2}{2} \right) - \alpha \left( \frac{L\gamma^2 \eta^2 }{2} +\gamma^2 \eta^3 L  + \frac{3 \gamma^2 \eta}{2} \right) \ge \frac{\gamma \eta}{4}.
\end{align*}
Also, the fact $\alpha \in (0,1]$ and our choice of $\gamma$ means
\begin{align*}
    8(1+\alpha) \kappa^2 \gamma + \frac{\alpha}{2} \le \alpha 
    \qquad \text{and}\qquad 
    (1+\alpha) 8 \kappa^2 \gamma + \frac{\alpha}{2} \le \alpha.
\end{align*}
Therefore, we obtain the optimization bound as 
\begin{align} \label{ieq:bound-P}
\small\begin{split}
\BE[\Psi_{t+1}] 
\le & \BE \left[ \Psi_t - \frac{\gamma \eta}{2} \Vert \nabla P(\bar x_t) \Vert^2 - \frac{\gamma \eta}{4} \Vert \bu_t \Vert^2 - \frac{\alpha \eta}{16m} \sum_{i=1}^m \Vert \vv_t'(i) \Vert^2  \right] \\
& + \BE \left[ \frac{\alpha \eta}{m} \sum_{i=1}^m \Vert \bar u_t - \nabla_x f(\bx_t, \vy_t(i) \Vert^2 + \frac{\alpha \eta}{m} \sum_{i=1}^m \Vert \vv_t(i) - \nabla_y f(\bx_t, \vy_t(i)) \Vert^2 + \frac{\alpha}{2 \eta m} \sum_{i=1}^m \Vert \by_t - \vy_t(i) \Vert^2 \right].
\end{split}
\end{align}
Note that it holds that
\begin{align*}
  \sum_{i=1}^m \Vert \mathbf{v}_t (i) - \nabla_y f(\bar x_t, \mathbf{y}_t(i)) \Vert^2 
\le  3m \Vert \bv_t -  \nabla_y f( \bx_t,\by_t) \Vert^2 + 3 \Vert \mathbf{v}_t  - \vone \bv_t \Vert^2 + 3L^2 \Vert \mathbf{y}_t - \mathbf{1} \bar y_t \Vert^2
\end{align*}
and
\begin{align*}
\sum_{i=1}^m \Vert \bar u_k - \nabla_x f(\bar x_t, \mathbf{y}_t(i) ) \Vert^2 
\le 2 m \Vert \bar u_t - \nabla_x f(\bar x_t, \bar y_t ) \Vert^2 + 2L^2 \Vert \mathbf{y}_t - \mathbf{1} \bar y_t \Vert^2,
\end{align*}
where we use the $L$-smoothness and Young's inequality. Then we combine Lemma \ref{lem:error-grad} to get
\begin{align} \label{plg:3}
\begin{split}
&\quad   \frac{\alpha \eta}{m} \sum_{i=1}^m \Vert \bar u_t - \nabla_x f(\bx_t, \vy_t(i) \Vert^2 + \frac{\alpha \eta}{m} \sum_{i=1}^m \Vert \vv_t(i) - \nabla_y f(\bx_t, \vy_t(i)) \Vert^2 \\
&\le\frac{6 \alpha \eta L^2}{m} \Vert \vy_t - \vone \by_t \Vert^2 + \frac{3 \alpha \eta}{m} \Vert \bar s_t - \nabla f(\bx_t, \by_t) \Vert^2 +   \frac{3 \alpha \eta}{m} \Vert \vv_t - \vone \bv_t \Vert^2 \\
&\le \frac{\alpha}{2\eta m} \Vert \vy_t - \vone \by_t \Vert^2 + 3 \alpha \eta \Vert \bar s_t - \nabla f(\bx_t, \by_t) \Vert^2 +   \frac{3 \alpha \eta}{m} \Vert \vv_t - \vone \bv_t \Vert^2,
\end{split}
\end{align}
where we use $\eta \le 1/(4L)$. Plugging  (\ref{plg:3}) into (\ref{ieq:bound-P}) and then use the inequality 
\begin{align*}
    \frac{1}{m} \sum_{i=1}^m \Vert \vv_t'(i) \Vert^2 \ge \Vert \bv_t' \Vert^2.
\end{align*}
Hence, we have
\begin{align*}
\BE[\Psi_{t+1}] &\le \BE \left[ \Psi_t - \frac{\gamma \eta}{2} \Vert \nabla P(\bar x_t) \Vert^2 - \frac{\gamma \eta}{4} \Vert \bu_t \Vert^2 - \frac{\alpha \eta}{16}  \Vert \bv_t' \Vert^2  \right] \\
&\quad + \BE \left[ 3 \alpha \eta \Vert \bar s_t - \nabla_x f(\bx_t, \by_t) \Vert^2 + \frac{3\alpha \eta}{m} \Vert \vv_t - \vone \bv_t \Vert^2 + \frac{\alpha}{\eta m}  \Vert \vy_t - \by_t \Vert^2 \right].
\end{align*}
Combining with Lemma \ref{lem:error-grad} and using
$\Vert \vy_t - \vone \by_t \Vert \le \Vert \vz_t - \vone \bz_t \Vert$ and $\eta \le 1/(4L)$, we obtain
\begin{align*}
\BE[\Psi_{t+1}] &\le \BE \left[ \Psi_t - \frac{\gamma \eta}{2} \Vert \nabla P(\bar x_t) \Vert^2 - \frac{\gamma \eta}{4} \Vert \bu_t \Vert^2 - \frac{\alpha \eta}{16} \Vert \bv_t' \Vert^2  \right] \\
&\quad + \BE \left[ \frac{3\alpha \eta}{m} \Vert \vv_t - \vone \bv_t \Vert^2 + 6 \alpha \eta \left \Vert \frac{1}{m} \sum_{i=1}^m  (\vg_t(i) - \nabla f_i(\vz_t(i))) \right \Vert^2 +  \frac{2\alpha}{\eta m}  \Vert \vz_t - \bz_t \Vert^2 \right].
\end{align*}
Note that
\begin{align*}
    \bu_t = \frac{\bx_t - \bx_{t+1}}{\gamma \eta}, \quad \bv_t' = \frac{\by_{t+1} - \by_t}{\eta} \quad\text{and} \quad \frac{1}{8 \gamma \eta} \ge \frac{\alpha}{16 \eta},
\end{align*}
then we have
\begin{align*} 
\BE[\Psi_{t+1}] &\le \BE \left[ \Psi_t - \frac{\gamma \eta}{2} \Vert \nabla P(\bar x_t) \Vert^2 - \frac{1}{8 \gamma \eta} \Vert \bx_{t+1} - \bx_t \Vert^2 - \frac{\alpha}{16 \eta} \Vert \bz_{t+1} - \bz_t \Vert^2  \right] \\
&\quad + \BE \left[ \frac{3\alpha \eta}{m} \Vert \vv_t - \vone \bv_t \Vert^2 + 6 \alpha \eta  \left \Vert \frac{1}{m} \sum_{i=1}^m  (\vg_t(i) - \nabla f_i(\vz_t(i))) \right \Vert^2  +  \frac{2\alpha}{\eta m}  \Vert \vz_t - \vone \bz_t \Vert^2 \right].
\end{align*}
Recalling the definition of $C_t$ and $U_t$, we obtain the result of Lemma \ref{lem:Lyp-P}.
\end{proof}

\section{The Proof of Lemma \ref{lem:Lyp-C}}

\begin{proof}
The relation of (\ref{update:avg}) means
\begin{align*}
& \Norm{\vx_{t+1} - \vone\bx_{t+1}} \\
=  & \rho \Norm{\vx_t - \gamma \eta\vu_t  - \vone(\bx_t - \eta\bu_t)} \\
\leq  & \rho\left(\Norm{\vx_t - \vone\bx_t} + \gamma \eta\Norm{\vu_{t} - \vone\bu_{t}}\right),
\end{align*}
where the last step is due to triangle inequality. Similarly, we define the notation~$\bar \Pi(\cdot) = \frac{1}{m} \vone \vone^\top (\cdot)$ for convenience. Then for variable $y$, we can verify that
\begin{align*}
& \Norm{\vy_{t+1} - \vone\by_{t+1}} \\
\leq  & \rho \Norm{\Pi(\vy_t + \eta\vv_t)  - \frac{1}{m}\vone\vone^\top\Pi\left(\vy_t + \eta\vv_t\right)} \\
\leq & \rho  \Norm{\Pi(\vy_t + \eta\vv_t)  - \Pi( \vone \by_t + \eta \vone \bv_t)} + \rho \Norm{\Pi( \vone \by_t + \eta \vone \bv_t) - \vone \bar \Pi\left(\vy_t + \eta\vv_t\right)} \\
\leq  & 2\rho \Norm{\vy_t + \eta\vv_t  - \vone(\by_t + \eta\bv_t)} \\
\leq  & 2\rho\left(\Norm{\vy_t - \vone\by_t} + \eta\Norm{\vv_{t} - \vone\bv_{t}}\right),
\end{align*}
where in the third inequality we use  the non-expansiveness of projection and Lemma 11 in ~\cite{ye2020decentralized}, i.e.
\begin{align*}
    \Norm{  \vone \bar \Pi(\vx) - \Pi(\vone \bx  )} \le \Vert \vx - \vone \bx \Vert.
\end{align*}
Consequently, we use Young's inequality together with $\gamma \in (0,1]$ to obtain
\begin{align} \label{recursion:z}
\begin{split}
&\quad \Norm{\vz_{t+1} - \vone\bz_{t+1}}^2 \\ 
&= \Norm{\vx_{t+1} - \vone \bx_{t+1}}^2 + \Norm{\vy_{t+1} - \vone \by_{t+1}}^2 \\
&\le 8 \rho^2 \Norm{\vy_t - \vone \by_t}^2 + 8 \rho^2 \eta^2 \Norm{\vv_t - \vone \bv_t}^2 +2 \rho^2 \Norm{\vx_t - \vone \bx_t}^2 + 2 \rho^2 \eta^2 \Norm{\vy_t - \vone \by_t}^2 \\
&\le 8\rho^2\Norm{\vz_t - \vone\bz_t}^2 + 8\rho^2\eta^2\Norm{\vs_{t} - \vone\bs_{t}}^2. 
\end{split}
\end{align}
Furthermore, if $24 \rho^2 \le 1$, we have
\begin{align} \label{ieq:zt12}
\begin{split}
&\quad    \Vert \vz_{t+1} - \vz_t \Vert^2 \\
&\le 3 \Vert \mathbf{z}_{t+1} - \mathbf{1} \bar z_{t+1} \Vert^2 + 3 \Vert \mathbf{z}_{t} - \mathbf{1} \bar z_{t} \Vert^2 + 3 \Vert \mathbf{1} \bar {z}_{t} - \mathbf{1} \bar z_{t+1} \Vert^2 \\
&\le (24 \rho^2 +3) \Vert \mathbf{z}_t - \mathbf{1} \bar z_k \Vert^2 + 24 \rho^2 \eta^2 \Vert \mathbf{s}_t - \mathbf{1} \bar s_t \Vert^2 + 3 \Vert \mathbf{1} \bar {z}_{t} - \mathbf{1} \bar z_{t+1} \Vert^2 \\
&\le 4 \Vert \mathbf{z}_t - \mathbf{1} \bar z_t \Vert^2 +  \eta^2 \Vert \mathbf{s}_t - \mathbf{1} \bar s_t \Vert^2 + 3m \Vert  \bar {z}_{t} - \bar z_{t+1} \Vert^2.
\end{split}
\end{align}
We let $\rho_t = \rho'$ for $\zeta_t=1$ and $\rho_t=\rho$ otherwise. 
The update of $\vg_{t+1}(i)$ means
\begin{align*} 
\begin{split}    
  & \BE\big[ \rho_t^2 \Norm{\vg_{t+1}(i)-\vg_t(i)}^2\big] \\
= & \rho'^2 p \BE\Norm{\dfrac{1}{b'}\sum_{\xi_{i,j}\in \fS_t'(i)}\!\!\!\!\nabla F_i(\vz_{t+1}(i);\xi_{i,j})-\vg_t(i)}^2 + \frac{\rho^2 (1-p)}{bq}\BE\Norm{\nabla F_i(\vz_{t+1}(i);\xi_{i,1}) - \nabla F_i(\vz_t(i);\xi_{i,1})}^2 \\
\leq & 3 \rho'^2 p\BE\Norm{\dfrac{1}{b'}\sum_{\xi_{i,j}\in \fS_t'(i)} \nabla F_i(\vz_{t+1}(i);\xi_{i,j})-\nabla f_i(\vz_{t+1}(i))}^2 + 3 \rho'^2 p\BE\Norm{\nabla f_i(\vz_{t+1}(i))-\nabla f_i(\vz_t(i))}^2 \\
& + 3 \rho'^2 p\BE\Norm{\nabla f_i(\vz_t(i))-\vg_t(i)}^2 + \frac{\rho^2(1-p)L^2}{bq  }\BE\Norm{\vz_{t+1}(i) - \vz_t(i)}^2 \\
\leq & \frac{3 \rho'^2 p\sigma^2}{b'} \BI[b' < n] + 3 \rho'^2 p L^2\BE\Norm{\vz_{t+1}(i)-\vz_t(i)}^2 
 \\ & + 3 \rho'^2 p\BE\Norm{\nabla f_i(\vz_t(i))-\vg_t(i)}^2 + \frac{\rho^2 (1-p)L^2}{bq}\BE\Norm{\vz_{t+1}(i) - \vz_t(i)}^2 \\
\le & \frac{3 \rho'^2 p\sigma^2}{b'} \BI[b' < n] + 3 \rho'^2 p\BE\Norm{\nabla f_i(\vz_t(i))-\vg_t(i)}^2 + \left(\frac{1-p}{bq} + 3p\right)\rho^2L^2\BE\Norm{\vz_{t+1}(i)-\vz_t(i)}^2,
\end{split}
\end{align*}
where the first inequality is based on update rules and Assumption~\ref{asm:SFO}; the second inequality is based on triangle inequality and the last inequality is due to Assumption~\ref{asm:smooth}. 
Summing over above result over $i=1,\dots,m$, obtain
\begin{align*}
   \BE\big[\rho_t^2 \Norm{\vg_{t+1}-\vg_t}^2\big] 
\leq \frac{3 \rho'^2 m  p\sigma^2}{b'} \BI[b' < n] + 3 \rho'^2 p\BE\Norm{\nabla \vf(\vz_t)-\vg_t}^2 + \left(\frac{1-p}{bq} + 3  p\right)\rho^2 L^2\BE\Norm{\vz_{t+1}-\vz_t}^2.
\end{align*}
Let $c = \max \left\{1/ (bq), 1 \right\}$ and plug it into  (\ref{ieq:zt12}), then
\begin{align}\label{inq:g_t-g_t+1}
\begin{split}
& \mathbb{E} [ \rho_t^2\Vert \mathbf{g}_{t+1} -  \mathbf{g}_t \Vert^2] \\
\le & \frac{3 \rho'^2 m  \sigma^2}{b'} \mathbb{I}[b' < n ] + 3 \rho'^2  \mathbb{E}[\Vert \nabla \mathbf{f}(\mathbf{z}_t) - \mathbf{g}_t \Vert^2] + c \rho^2 L^2 \mathbb{E}[ \Vert \mathbf{z}_{t+1} - \mathbf{z}_t \Vert^2]\\ 
\le & \frac{3 \rho'^2 m  \sigma^2}{b'} \mathbb{I}[b' < n ] + 3 \rho'^2 \mathbb{E}[\Vert \nabla \mathbf{f}(\mathbf{z}_t) - \mathbf{g}_t \Vert^2] \\
& + 16c \rho^2 L^2  \mathbb{E}[ \Vert \mathbf{z}_{t} - \mathbf{1} \bar z_{t} \Vert^2] + 4c \rho^2 L^2 \eta^2 \mathbb{E}[\Vert \mathbf{s}_t - \mathbf{1} \bar s_t \Vert^2 ] + 12c \rho^2 m L^2 \mathbb{E}[\Vert \bar z_{t+1} - \bar z_t \Vert^2].
\end{split}
\end{align}
Furthermore,  we have
\begin{align}\label{inq:s_t+1}
\begin{split}
   & \Norm{\vs_{t+1} - \vone\bs_{t+1}} \\
\leq & \rho_t\Norm{\vs_t + \vg_{t+1} - \vg_t - \frac{1}{m}\vone\vone^\top(\vs_t + \vg_{t+1} - \vg_t)} \\
\leq & \rho_t\Norm{\vs_t - \vone\bs_t} + \rho_t\Norm{\vg_{t+1} - \vg_t - \frac{1}{m}\vone\vone^\top(\vg_{t+1} - \vg_t)} \\
\leq & \rho_t\Norm{\vs_t - \vone\bs_t} + \rho_t\Norm{\vg_{t+1} - \vg_t},
\end{split}
\end{align}
where the second inequality is based on triangle inequality and the last step uses Lemma~\ref{lem:avg-norm}. Combining the results of (\ref{inq:g_t-g_t+1}) and~(\ref{inq:s_t+1}) and using $\eta \le 1/(4L)$, we have
\begin{align}\label{recursive:s}
\begin{split}
& \quad \eta^2 \BE\Norm{\vs_{t+1} - \vone\bs_{t+1}}^2 \\
&\leq  2\rho^2 \eta^2 \Norm{\vs_t - \vone\bs_t}^2 + 2\eta^2 \BE [ \rho_t^2\Norm{\vg_{t+1} - \vg_t}^2] \\
&\le 4c \rho^2 \eta^2 \mathbb{E}[ \Vert \mathbf{s}_t - \mathbf{1} \bar s_t \Vert^2 ] + 4 c\rho^2  \mathbb{E}[\Vert \mathbf{z}_t - \mathbf{1} \bar z_t \Vert^2] \\
&\quad + 6 \rho'^2 \eta^2 \mathbb{E}[ \Vert \nabla \mathbf{f}(\mathbf{z}_t) - \mathbf{g}_t \Vert^2 ] +  2  c\rho^2 m \mathbb{E}[ \Vert \bar z_{t+1} - \bar z_t \Vert^2] + \frac{6 \rho'^2m \eta^2 \sigma^2}{b'} \mathbb{I}[b' < n ].
\end{split}
\end{align}
Combining (\ref{recursive:s}) and (\ref{recursion:z}), we obtain
\begin{align*}
&\quad \mathbb{E} \big[ \Vert \mathbf{z}_{t+1} - \mathbf{1} \bar z_{t+1} \Vert^2 + \eta^2 \Vert \mathbf{s}_{t+1} - \mathbf{1} \bar s_{t+1} \Vert^2\big] \\
&\le 12 c\rho^2 \mathbb{E}[\Vert \mathbf{z}_t - \mathbf{1} \bar z_t \Vert^2 +  \eta^2 \Vert \mathbf{s}_t - \mathbf{1} \bar s_t \Vert^2 ] 
 + 6 \rho'^2 \eta^2 \mathbb{E}[\Vert \nabla \mathbf{f}(\mathbf{z}_t ) - \mathbf{g}_t \Vert^2]   +2c\rho^2 m \mathbb{E}[\Vert \bar z_{t+1} - \bar z_t \Vert^2  ] +\frac{ 6 \rho'^2 m \eta^2  \sigma^2}{b'}  \mathbb{I}[b'\!<\!n ],
\end{align*}
which finishes our proof. 
\end{proof}

\section{The Proof of Lemma \ref{lem:Lyp-V}}

\begin{proof}
The update of $\vg_{t+1}(i)$ means
\begin{align*}
\begin{split}  
& \BE\Norm{\vg_{t+1}(i)-\nabla f_i(\vz_{t+1}(i))}^2 \\
= & p\BE\Norm{\dfrac{1}{b'}\sum_{\xi_{i,j}\in \fS'_t(i)}\nabla F_i(\vz_{t+1}(i);\xi_{i,j})-\nabla f_i(\vz_{t+1}(i))}^2 \\
& + (1-p)\BE\Norm{\vg_t(i) + \dfrac{\omega_t(i)}{bq}\sum_{\xi_{i,j}\in \fS_t(i)}\big(\nabla F_i(\vz_{t+1}(i);\xi_{i,j}) - \nabla F_i(\vz_t(i);\xi_{i,j})\big)-\nabla f_i(\vz_{t+1}(i))}^2 \\
\leq & \frac{p\sigma^2}{b'} \BI[b' < n ] + (1-p) \BE \Big \Vert \vg_t(i) - \nabla f_i(\vz_t(i))  \\
&\quad + \dfrac{\omega_t(i)}{bq}\sum_{\xi_{i,j}\in \fS_t(i)}\big(\nabla F_i(\vz_{t+1}(i);\xi_{i,j}) - \nabla F_i(\vz_t(i);\xi_{i,j})-\nabla f_i(\vz_{t+1}(i))+\nabla f_i(\vz_t(i))\big) \Big \Vert^2 \\
= & \frac{p\sigma^2}{b'} \BI[b'<n] + (1-p)\BE\Norm{\vg_t(i) - \nabla f_i(\vz_t(i))}^2 \\
& + (1-p)\BE\Norm{\dfrac{\omega_t(i)}{bq}\sum_{\xi_{i,j}\in \fS_t(i)}\big(\nabla F_i(\vz_{t+1}(i);\xi_{i,j}) - \nabla F_i(\vz_t(i);\xi_{i,j})-\nabla f_i(\vz_{t+1}(i))+\nabla f_i(\vz_t(i))\big)}^2 \\
\leq & \frac{p\sigma^2}{b'} \BI[b'<n]+ (1-p)\BE\Norm{\vg_t(i) - \nabla f_i(\vz_t(i))}^2 + \dfrac{1-p}{bq}\BE\Norm{\nabla F_i(\vz_{t+1}(i);\xi_{i,j}) - \nabla F_i(\vz_t(i);\xi_{i,j})}^2 \\
\leq & \frac{p\sigma^2}{b'} \BI[b'<n] + (1-p)\BE\Norm{\vg_t(i) - \nabla f_i(\vz_t(i))}^2 + \dfrac{(1-p)L^2}{bq}\BE\Norm{\vz_{t+1}(i)-\vz_t(i)}^2,
\end{split}
\end{align*}
where the first inequality is based on  Assumption~\ref{asm:SFO}; the second inequality use the property of variance; the last inequality is based on Assumption~\ref{asm:smooth}; the second equality uses the property of martingale according to Proposition~1 in \citep{fang2018spider}.
Taking the average over $i=1,\dots,m$ for above result and using (\ref{ieq:zt12}), we obtain
\begin{align}\label{ieq:error-grad2}
\begin{split}
 &\BE\Norm{\vg_{t+1}-\nabla \vf(\vz_{t+1})}^2 \\
\leq & \frac{m p\sigma^2}{b'} \BI[b' <n ]+ (1-p)\BE\Norm{\vg_t - \nabla \vf(\vz_t)}^2 + \dfrac{(1-p)L^2}{ b q}\BE\Norm{\vz_{t+1}-\vz_t}^2 \\
\leq & \frac{m p\sigma^2}{b'} \BI[b' <n ]+ (1-p)\BE\Norm{\vg_t - \nabla \vf(\vz_t)}^2 \\
& +  \frac{(1-p)L^2}{bq} \mathbb{E}[4 \Vert \mathbf{z}_t - \mathbf{1} \bar z_t \Vert^2 +  \eta^2 \Vert \mathbf{s}_t - \mathbf{1} \bar s_t \Vert^2 + 3 m\Vert  \bar {z}_{t+1} -  \bar z_{t} \Vert^2],
\end{split}
\end{align}
which is the variance bound as claimed by the definition of $V_t$.
\end{proof}

\section{The Proof of Lemma \ref{lem:Lyp-U}}

We omit the detailed proof since it is almost identical to the proof of Lemma \ref{lem:Lyp-V}. We leave this as an exercise to the reader.
Compared with Lemma \ref{lem:Lyp-V}, the quantities in Lemma \ref{lem:Lyp-U} can be scaled with an additional factor of $1/m$ by using 
the fact
\begin{align*}
\BE \Big\Vert \sum_{i=1}^m a_i  \Big\Vert^2 = \sum_{i=1}^m \BE \Vert a_i \Vert^2,  
\end{align*}
where each $a_1,\dots,a_m$ are independent with zero mean. 

\section{The Proof of Theorem \ref{thm:main}}

\begin{proof}
Combing Lemma \ref{lem:Lyp-P}, \ref{lem:Lyp-C}, \ref{lem:Lyp-V} and \ref{lem:Lyp-U} together, we obtain
\begin{align*}
& \BE[\Phi_{t+1}] \\
\le & \BE \left[ \Phi_t   - \frac{\gamma \eta}{2} \Vert \nabla P(\bar x_t) \Vert^2 - \frac{1}{8 \gamma \eta} \Vert \bx_{t+1} - \bx_t \Vert^2 - \frac{\alpha}{16 \eta} \Vert \bz_{t+1} - \bz_t \Vert^2 + 6 \alpha \eta U_t + \frac{3 \alpha }{m \eta} C_t \right] \\
& + \frac{\eta }{mp} \BE \left[ -p V_t + \frac{4(1-p) L^2}{ m b q} C_t + \frac{3(1-p)L^2 }{b q} \Vert \bz_{t+1} - \bz_t \Vert^2 + \frac{p \sigma^2}{b'} \BI[b' < n ]  \right] \\
& + \frac{\eta}{p} \BE \left[ - p U_t + \frac{4 (1-p) L^2}{m^2 b q} C_t + \frac{3 (1-p) L^2}{m b q} \Vert \bz_{t+1} - \bz_t \Vert^2 + \frac{p \sigma^2}{m b'} \BI[b' < n ] \right]  \\
& + \frac{1}{\eta m} \BE\left[-(1-12c \rho^2)  C_t + 6  m \rho'^2 \eta^2 V_t   +2c\rho^2 m \mathbb{E}[\Vert \bar z_{t+1} - \bar z_t \Vert^2  ] +\frac{ 6 \rho'^2 m \eta^2  \sigma^2}{b'}  \mathbb{I}[b' < n ] \right] \\
=&  \Phi_t + \BE\left[   - \frac{\gamma \eta}{2} \Vert \nabla P(\bar x_t) \Vert^2 - \frac{1}{8 \gamma \eta} \Vert \bx_{t+1} - \bx_t \Vert^2 - \left(  \frac{\alpha}{16 \eta} - \frac{2c \rho^2}{\eta} - \frac{6 \eta (1-p) L^2}{m b p q}  \right) \Vert \bz_{t+1} - \bz_t \Vert^2 \right] \\
& - (1 - 6 \alpha) \eta U_t -  \frac{(1- 6 m \rho'^2)\eta}{m} V_t - \left( \frac{1 - 12 c\rho^2 - 3 \alpha }{\eta m} - \frac{8(1-p) L^2}{m^2 b p q} \right) C_t \\
&\quad  + \left(   6 \rho'^2 \frac{2 }{m} \right) \frac{\eta \sigma^2}{b'} \BI[b' <n].
\end{align*}
Plugging in our setting of parameters in (\ref{para:val}),  it can be seen that
\begin{align*}
\BE[\Phi_{t+1}] \le \Phi_t -  \frac{\gamma \eta}{2} \BE[ \Vert \nabla P(\bar x_t) \Vert^2]   - \frac{4 \alpha}{\eta m} C_t + \frac{3 \eta \sigma^2}{m b'} \BI[b' <n]. 
\end{align*}
Telescoping for $t = 0,1\cdots,T-1$, we obtain
\begin{align} \label{ieq:res-bound}
\frac{1}{T} \sum_{t=0}^{T-1} \BE[\Vert \nabla P(\bx_t) \Vert^2] \le \frac{2 }{\gamma \eta T} \Phi_0 - \frac{8 \alpha}{\gamma \eta^2 m T} \sum_{t=0}^{T-1} C_t + \frac{6 \sigma^2}{\gamma m b'} \BI[b' <n].
\end{align}
Note that achieving $\bx_t$ is not simple, so the output $x_{\rm out}$ is sampled from $\{\vx_t(i)\}$ where $t=0,\dots,T-1$ and $i=1,\dots,m$. We also has the following bound:
\begin{align*}
\BE\Norm{\nabla P(x_{\rm out})}^2 = & \frac{1}{mT}\sum_{i=1}^{m}\sum_{t=0}^{T-1}\Norm{\nabla P(\vx_t(i))}^2 \\
\leq & \frac{2}{mT}\sum_{i=1}^{m}\sum_{t=0}^{T-1} \left(\Norm{\nabla P(\bar x_t)}^2 + \Norm{\nabla P(\vx_t(i))-\nabla P(\bx_t)}^2\right) \\
\leq & \frac{2}{mT}\sum_{i=1}^{m}\sum_{t=0}^{T-1} \left(\Norm{\nabla P(\bar x_t)}^2  + L^2\Norm{\vx_t(i)-\bx_t}^2\right) \\
= & \frac{2}{T}\sum_{t=0}^{T-1} \Norm{\nabla P(\bar x_t)}^2 + \frac{2L^2}{mT}\sum_{t=0}^{T-1}\Norm{\vx_t-\vone\bx_t}^2 \\
\leq & \frac{2}{T}\sum_{t=0}^{T-1} \Norm{\nabla P(\bar x_t)}^2 + \frac{2L^2}{mT}\sum_{t=0}^{T-1} C_t,
\end{align*}
where the first step use Young's inequality; the second inequality is due to Assumption~\ref{asm:smooth}. Now we plug in (\ref{ieq:res-bound}) to get the following bound as
\begin{align*}
\BE\Vert \nabla P( x_{\rm out}) \Vert^2 &\le \frac{4 }{\gamma \eta T} \Phi_0 - \left( \frac{16 \alpha}{\gamma \eta^2 } - 2L^2 \right) \frac{1}{mT} \sum_{t=0}^{T-1} C_t + \frac{12 \sigma^2}{\gamma m b'} \BI[b'<n].
\end{align*}
Recall the choice of $\gamma$ in (\ref{choice:lambda}); $\eta \le 1/(4L)$ and $p T \ge 2$. Hence, we have
\begin{align*}
\BE\Vert \nabla P(x_{\rm out}) \Vert^2 &\le \frac{4 }{\gamma \eta T} \Phi_0  + \frac{12 \sigma^2}{\gamma m b'} \BI[b'<n] \\
&= \frac{4 }{\gamma \eta T} \left( \Psi_0 + \frac{\eta}{m p } V_0 + \frac{\eta}{p} U_0 + \frac{\eta}{m} C_0 \right) +\frac{12 \sigma^2}{\gamma m b'} \BI[b'<n] \\
&= \frac{4 }{\gamma \eta T} \left( \Psi_0 + \frac{2\eta \sigma^2}{ mb' p} \BI[b' < n ] + \frac{\eta}{m} \Vert \vs_0 - \vone \bs_0 \Vert^2  \right) + \frac{12 \sigma^2}{\gamma m b'} \BI[b'<n]  \\
&\le \frac{4 }{\gamma \eta T} \left( \Psi_0 + \frac{2\eta \sigma^2}{ mb' p} \BI[b' < n ] + \frac{\eta \rho_0^2}{m} \Vert \vg_0 - \vone \bg_0 \Vert^2  \right) + \frac{12 \sigma^2}{\gamma m b'} \BI[b'<n] \\
&\le \frac{8 }{\gamma \eta T} \Psi_0 + \frac{16 \sigma^2}{\gamma m b' } \BI[b'<n] + \frac{4 \rho_0^2 \Vert \vg_0 - \vone \bg_0 \Vert^2 }{\gamma m}.
\end{align*}
Therefore the parameters in (\ref{para:val}) and (\ref{choice:b}) guarantee that  $\BE \Vert \nabla P(x_{\rm out}) \Vert^2 \le \epsilon^2$. The Jensen's inequality further implies that the output is a nearly stationary point satisfying $\BE \Vert \nabla P(x_{\rm out}) \Vert \le \epsilon$.

Recall our choice of $\gamma $ in (\ref{choice:lambda}), we know that $1/\gamma = \Theta(\kappa^2)$. Then the total SFO complexity for all agents in expectation is
\begin{align*}
mb' + mT(b'p + b q (1-p))  = mb' + \frac{2mT b' b q}{b' + b q } \le mb' + 2mT b q. 
\end{align*}
Plug in the choice of $b,b',q$ yields the SFO complexity as claimed. Next, recalling the definition of $\rho,\rho', \rho_0$ in (\ref{dfn:rho}), we know the total number of communication rounds is 
\begin{align*}K_0 + T(p K' + (1-p)K)  = 
\begin{cases}
    \fO\left({\kappa^2 L\eps^{-2}}/\sqrt{\delta} \right), & b' \ge m; \\[0.2cm]
    \fO\left({\kappa^2 L\eps^{-2} \log (m/b')}/\sqrt{\delta}\right) , & b' < m.
\end{cases}
\end{align*}
\end{proof}

\section{Future Directions and Subsequent Works}

A future direction is to establish decentralized stochastic algorithms with better dependency on the condition number $\kappa$ by devising multiple-looped algorithms such as the single-machine setting \citep{zhang2022sapd+}.
It is also interesting to consider decentralized minimax optimization with the nonconvex-non-strongly-concave objectives ~\citep{lin2020gradient,lin2020near,zhang2020single,xu2023unified}, 
or some classes of nonconvex-nonconcave objectives~\citep{zheng2022doubly,li2022nonsmooth,jin2020local,diakonikolas2021efficient,yang2020global,chen2022faster,guo2020fast}.

After we posted our work on arXiv, we also noticed that some subsequent works studying similar problem setups~\citep{xu2023decentralized,huang2023near,mancino2023variance}. Some of the results in these works can apply to more general setups than our problem, but the convergence rates in these works do not surpass the results of our paper. (a) Although \citet{xu2023decentralized} considers more general regularizers (we only consider the case that regularizer in $y$ is the indicator function of set $\fY$), \citet{xu2023decentralized} only studies the offline case and does not study stochastic optimization. 
Following our notations, the method proposed by \citet{xu2023decentralized} requires the computation complexity of $\mathcal{O}(N \kappa^{2.5} \epsilon^{-2})$, which is worse than our result of $\mathcal{O}(N+\sqrt{N} \kappa^2 \epsilon^{-2})$. (b) Compared with \citet{xu2023decentralized}, \citet{mancino2023variance} additionally considers the stochastic optimization, but only for the offline case. It requires the computation complexity of $\mathcal{O}(N + \sqrt{m N} \kappa^2 \epsilon^{-2})$, which is still worse than our $\mathcal{O}(N + \sqrt{N} \kappa^2 \epsilon^{-2})$.
(c) \citet{huang2023near} considers the unconstrained online case under Polyak--Łojasiewicz (PL) condition. 
In fact, all of the analyses in our paper still hold if we relax the strong concavity assumption to the PL condition in the unconstrained case.
Furthermore, although \citet{huang2023near}'s computation complexity $\mathcal{O}(\kappa^3 \epsilon^{-3})$ matches our result, their communication complexity $\mathcal{O}(\kappa^3 \epsilon^{-3})$ is worse than our $\tilde{\mathcal{O}}(\kappa^2 \epsilon^{-2})$.